%% file: main.tex
\documentclass{msml2020}
\usepackage{times}
\usepackage[utf8]{inputenc}
\usepackage[english]{babel}
\usepackage{bm}
\usepackage{mathrsfs}
\usepackage{amsmath,amsfonts,amssymb}
\usepackage{graphicx}
\usepackage{epstopdf}
\usepackage{algorithm}
\usepackage{algpseudocode}
\usepackage{soul}
\usepackage{enumitem}
\usepackage{multirow}

\newcommand{\edit}[1]{\textcolor{black}{#1}}

\title{Robust Training and Initialization of Deep Neural Networks: \\ An Adaptive Basis Viewpoint}

\msmlauthor{%
 \Name{Eric C. Cyr} \Email{eccyr@sandia.gov}\\
 \addr Center for Computing Research, Sandia National Laboratories
 \AND
 \Name{Mamikon A. Gulian} \Email{mgulian@sandia.gov}\\
 \addr Center for Computing Research, Sandia National Laboratories
 \AND
 \Name{Ravi G. Patel} \Email{rgpatel@sandia.gov}\\
 \addr Center for Computing Research, Sandia National Laboratories
 \AND
 \Name{Mauro Perego} \Email{mperego@sandia.gov}\\
 \addr Center for Computing Research, Sandia National Laboratories
 \AND
 \Name{Nathaniel A. Trask} \Email{natrask@sandia.gov}\\
 \addr Center for Computing Research, Sandia National Laboratories
}

\newcommand{\sym}[2]{(\protect\includegraphics[height=8pt]{figures/symbols/#1_#2})}

\begin{document}

\maketitle
\begin{abstract}

Motivated by the gap between theoretical optimal approximation rates of deep neural networks (DNNs) and the accuracy realized in practice, we seek to improve the training of DNNs. The adoption of an adaptive basis viewpoint of DNNs leads to novel initializations and a hybrid least squares/gradient descent optimizer. We provide analysis of these techniques and illustrate via numerical examples dramatic increases in accuracy and convergence rate for benchmarks characterizing scientific applications where DNNs are currently used, including regression problems and physics-informed neural networks for the solution of partial differential equations.
\end{abstract}

\section{Introduction}

Universal approximation properties of neural networks are often touted as an explanation of the success of deep neural networks (DNNs) in applications. 
Despite their importance, such theorems offer no explanation for the advantages of neural networks, let alone \textit{deep} neural networks, over classical approximation methods, since universal approximation properties are enjoyed by polynomials \citep{cheney2009course} as well as single layer neural networks \citep{cybenko1989approximation}. 
To address this, a recent thread has emerged in the literature concerning optimal approximation with deep ReLU networks, where the error in an optimal choice of weights and biases is bounded from above using the width and depth of the neural network. 

For example, using the ``sawtooth'' function of \citet{telgarsky2015representation}, \citet{yarotsky2017error} constructed an exponentially accurate (in the number of layers) ReLU network emulator for multiplication $(x, y) \mapsto xy$. This construction is used to obtain upper bounds on optimal approximation based upon DNN emulation of polynomial approximation.
Building on these ideas, \citet{opschoor2019deep} proved that optimal approximation with deep ReLU networks can emulate adaptive $hp$-finite element approximation, with greater depth allowing $p$-refinement to obtain exponential convergence rates. 
An additional contribution by \citet{he2018relu} reinterpreted single hidden layer ReLU networks as $r$-adaptive piecewise linear finite element spaces.

Despite this, it remains a challenge to realize 
these theorized convergence rates for DNNs using practical initialization and training methods. The need is particularly acute
 in scientific machine learning (SciML) applications which demand greater accuracy and robustness from DNNs \citep{raissi2019physics, baker2019workshop}. 
In practice, optimization and initialization challenges preclude the realization of theoretical convergence rates. Optimizers are susceptible to finding suboptimal local minima of loss functionals, and as a result DNN regression typically stagnates after achieving only a few digits of accuracy. For example, using the aforementioned architecture of \citet{yarotsky2017error} for the deep ReLU emulator of $x \mapsto x^2$\edit{, but with random initial weights, \citet{fokina2019growing} showed that training with stochastic gradient descent to approximate $x \mapsto x^2$ fails to demonstrate a significant improvement in error with depth, let alone exponential convergence with the number of layers.} \citet{lu2018collapse, lu2019dying} demonstrate consistent failure of deep ReLU networks to approximate the function $|x|$ on $[-1,1]$ due to gradient death at initialization.
These results illustrate the need for robust training and initialization algorithms for regression and approximation in scientific problems. We aim to bridge the gap between theoretical optimal error estimates and the error one can consistently achieve with training.   
\begin{figure}[htbp!]
    \centering
    \includegraphics[width=0.66\textwidth]{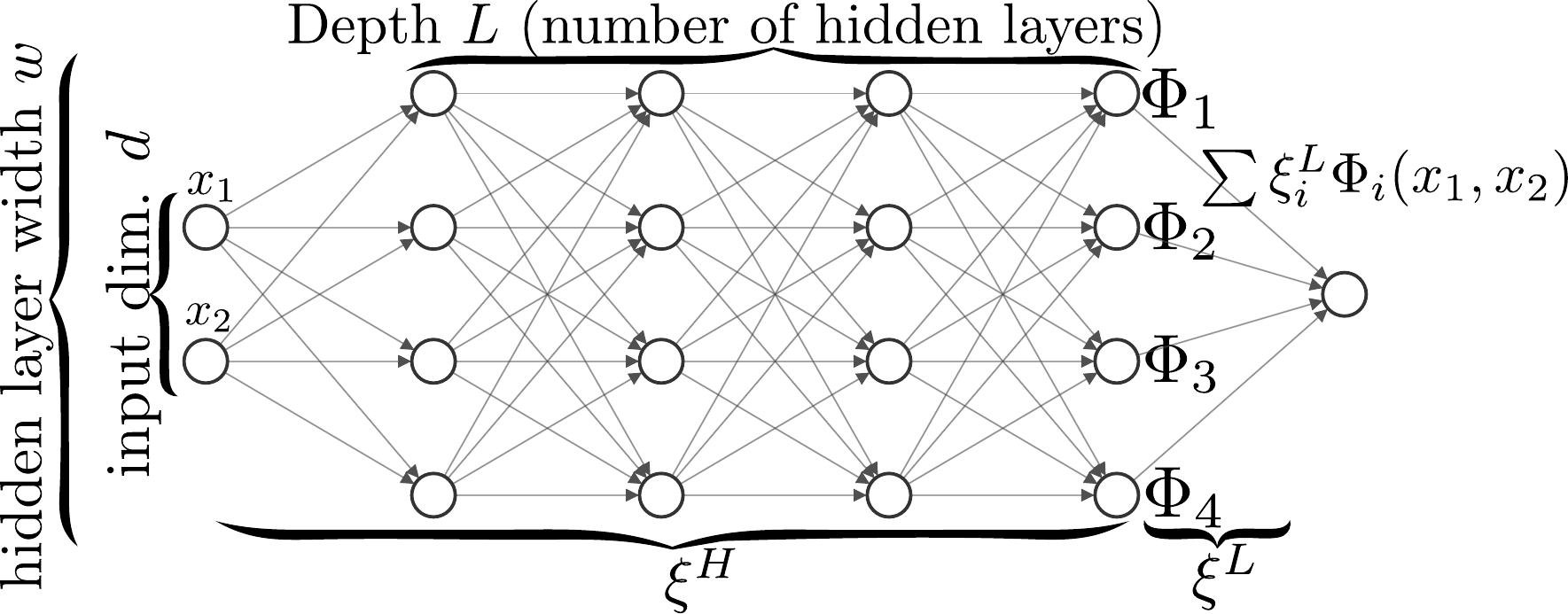}
    \vspace{-2ex}
    \caption{Adaptive basis view of a DNN with linear output layer, with notation used in this article.}
    \label{fig:basisComparison}
\end{figure}

In the current work we adapt the perspective that DNNs provide a meshfree technique to construct an adaptive basis.
This viewpoint suggests an optimizer that alternates between least squares (LS) and gradient descent (GD) steps. 
This process amounts to adapting the basis functions with respect to  to the data using GD while ensuring with LS that basis functions optimally fit the data. 
This training strategy is applicable to networks with arbitrary activation function in the hidden layers, a final linear activation layer, and a mean-square loss functional. 

From the adaptive basis viewpoint, we also propose a new initialization for deep ReLU networks that we refer to as the ``box initialization'', designed to provide an expressive initial guess for the basis. We show that this initialization outperforms the Glorot \citep{glorot10} and He initializations \citep{he2015delving} for one-dimensional approximation using plain networks for a mild number of layers. Via a novel analysis of DNNs in terms of this adaptive basis perspective, we extend the box initialization to residual ReLU networks and show improvements upon the He and Glorot initializations through 256 layers.

Combining our initialization for residual neural networks (ResNets) with the hybrid LSGD training algorithm, we demonstrate convergence of the approximation error for very deep ReLU neural networks with increasing depth. While the variance in errors remains high and the ``convergence rates'' are lower than suggested theoretically, the improvement in reliability across a range of regression-like applications is substantial. Further, the architectures used are standard. In contrast, previous works in the literature have focused on how to prevent collapse or else considered specialized architectures. 

\section{Problem statement}

We consider in this work the following class of $\ell_2$ regression problems:
\begin{equation}\label{eqn:loss}
    \underset{\bm{\xi}}{\text{argmin }} 
    \sum_{k=1}^K
    \epsilon_k
    \left \| \mathcal{L}_k[u] - \mathcal{L}_k\left[\mathcal{NN}_{\bm{\xi}}\right] \right \|^2_{\ell_2(\mathcal{X}_k)}
\end{equation}
where for each $k = 1, 2, ..., K$, $\mathcal{X}_k = \{x^{(k)}_i\}_{i=1}^{N_k}$ denotes a finite collection of data points, $\mathcal{NN}_{\bm{\xi}}$ a neural network with parameters $\bm{\xi}$, and $\mathcal{L}_k$ a linear operator. In the case where $k=1$ and $\mathcal{L}$ is the identity, we obtain the standard regression problem
\begin{equation}\label{eqn:std-regression}
    \underset{\bm{\xi}}{\text{argmin }} 
    \| u - \mathcal{NN}_{\bm{\xi}} \|_{\ell_2(\mathcal{X})}^2.
\end{equation}
In general \eqref{eqn:loss} represents a broader class of multi-term loss functions, including those used in physics-informed neural networks (\cite{raissi2019physics}) for solving linear PDEs (see Section \ref{sec:pinns}). Moreover, while we restrict our study to a single scalar ``target'' function $u$ in most of the paper,  
in Section \ref{sec:multifunction} we apply our framework to regress multiple functions simultaneously.

We consider the family of neural networks $\mathcal{NN}_{\bm{\xi}}: \mathbb{R}^d \rightarrow \mathbb{R}$ consisting of $L$ hidden layers of width $w$ composed with a final linear layer (see Fig. \ref{fig:basisComparison}), admitting the representation
\begin{equation}\label{eqn:NNbasis}
\mathcal{NN}_{\bm{\xi}}(\bm{x}) = \sum \limits_{i=1}^{w} \xi_i^{\text{L}} \Phi_i(\bm{x};\bm{\xi}^{\text{H}})
\end{equation}
where $\bm{\xi}^{\text{L}}$ and $\bm{\xi}^{\text{H}}$ are the parameters corresponding to the final linear layer and the hidden layers respectively, and we interpret $\bm{\xi}$ as the concatenation of $\bm{\xi}^{\text{L}}$ and $\bm{\xi}^{\text{H}}$. Working with this form allows us to highlight the interpretation of neural networks as an adaptive basis.

A broad range of architectures admit this interpretation. In this work we consider both plain neural networks (also referred to as multilayer perceptrons) and residual neural networks (ResNets). Defining the affine transformation, $\bm{T}_l(\bm{x},\bm{\xi}) = \bm{W}_l^{\bm{\xi}} \cdot \bm{x} + \bm{b}_l^{\bm{\xi}}$, and given an activation function $\sigma$, plain neural networks correspond to the choice 

\begin{equation}
    \bm{\Phi}^{\text{plain}}(\bm{x},\bm{\xi}) = \bm{\sigma}\circ\bm{T}_L\circ \dots \circ \bm{\sigma}\circ\bm{T}_1,
\end{equation}
while residual networks (see~\cite{he2016deep,he2016identity}) correspond to
\begin{equation}
    \bm{\Phi}^{\text{res}}(\bm{x},\bm{\xi}) = (\bm{I}+\bm{\sigma}\circ\bm{T}_L)\circ \dots \circ (\bm{I}+\bm{\sigma}\circ\bm{T}_2)\circ (\bm{\sigma}\circ\bm{T}_1),
\end{equation}
where $\bm{\Phi}$ is the vector of the $w$ functions $\Phi_i$, $\bm{\sigma}$ the vector of the $w$ activation functions $\sigma$ and $\bm{I}$ denotes the identity. In both cases $\bm{\xi}^{\text{H}}$ corresponds to the weights and biases $\bm{W}$ and $\bm{b}$.

In the case of a single hidden layer plain network with ReLU activation, one obtains a piecewise linear $C^0$ finite element space. This case has been considered by \citet{he2018relu}, who show that training amounts to adapting a piecewise linear finite element space to data. In the broader context considered here, an adaptive basis tailored to the choice of activation function is obtained. For example, selecting a radial basis function (RBF) as activation for a single layer network corresponds to a RBF space with centers and shape parameters adapted to data.  Many other architectures admit the proposed interpretation, such as e.g. convolutional networks. 

\section{Hybrid least squares/GD training approach}
\label{LS_section}
Using the Neural Network representation in \eqref{eqn:NNbasis}, equation \eqref{eqn:loss} reads
\begin{equation} \label{eqn:separableNonLinearLeastSquare}
    \underset{\bm{\xi^\text{L}},\,\bm{\xi^\text{H}}}{\text{argmin }} 
    \sum_{k=1}^K \epsilon_k
    \left \| \mathcal{L}_k [u] -
 \sum_i \xi_i^{\text{L}} \mathcal{L}_k \left[\Phi_i(\bm{x},\bm{\xi}^{\text{H}})\right]
\right \|_{\ell_2(\mathcal{X}_k)}^2.
\end{equation}
A typical approach to solving Equation \ref{eqn:separableNonLinearLeastSquare} is to apply gradient descent with backpropagation jointly in $(\bm{\xi}^{\text{L}}, \bm{\xi}^{\text{H}})$. 
Given the adaptive basis viewpoint, an alternative is to
 hold the hidden weights $\bm{\xi}^{\text{H}}$ constant and minimize w.r.t. to $\bm{\xi}^{\text{L}}$, yielding the LS problem (for simplicity focusing on $K=1$):
\begin{equation}\label{eqn:lsq}
\underset{\bm{\xi}^{\text{L}}}{\text{argmin }}
    \left \| A \bm{\xi}^{\text{L}} - \bm{b} \right \|^2_{\ell_2(\mathcal{X})}
\end{equation}
Here we have $\bm{b}_i = \mathcal{L}[u](\bm{x}_i)$ and $A_{ij} = \mathcal{L} \left[\Phi_j(\bm{x}_i,\bm{\xi}^{\text{H}})\right]$ for $\bm{x}_i\in \mathcal X, \; i=1,\ldots, N$, $j=1, \ldots, w$. Problem \ref{eqn:lsq} is well posed if $N \ge w$ and $A$ is a full-rank matrix; otherwise the problem is under-determined and admits multiple solutions. This occurs if the basis functions $\Phi_j$ are linearly dependent over $\ell_2(\mathcal{X})$, as can occur for many weights initializations (see Section \ref{sec:init}). In that case, the Moore-Penrose pseudo-inverse $A^{+}$ can be used to compute the minimum-norm solution $\bm{\xi}^{\text{L}} = A^{+} \bm{b}$. In this work, we use the TensorFlow \citep{tensorflow2015} implementation provided by the function \texttt{lstsq} to compute the minimun-norm solution $\bm{\xi}^{\text{L}}$.

Exposing the LS problem in this way prompts a natural modification of gradient descent. The optimization algorithm proceeds by alternating between: a LS solve to update $\bm{\xi}^\text{L}$ by a global minimum for given $\bm{\xi}^{H}$; 
and a GD step to update $\bm{\xi}^\text{H}$ (Algorithm \ref{alg:LSGD}).

\begin{algorithm}[]
   \caption{Hybrid least squares/gradient descent}
    \begin{algorithmic}[1]
      \Function{LSGD}{$\bm{\xi}_0^H$}
        \State $\bm{\xi}^H = \bm{\xi}_0^H$\Comment{Input initialized hidden parameters}\label{lne:alg1-init}
        \State $\bm{\xi}^L = LS(\bm{\xi}^H)$\Comment{Solve LS problem for $\bm{\xi}^L$}
        \For{$i = 1\dots$}
            \State $\bm{\xi}^H = GD(\bm{\xi})$\Comment{Solve GD problem}
            \State $\bm{\xi}^L = LS(\bm{\xi}^H)$
        \EndFor
       \EndFunction
\end{algorithmic}
\label{alg:LSGD}
\end{algorithm}

Problem \ref{eqn:separableNonLinearLeastSquare} is referred to in the inverse-problems literature as a \emph{separable nonlinear least square} problem. It is often solved with the variable projection method \citep{golubPereyra_1973, golub_2003} in which $\bm{\xi}^{\text{L}}$ is computed by solving \eqref{eqn:lsq} as a function of $\bm{\xi}^{\text{H}}$ and is substituted into \eqref{eqn:separableNonLinearLeastSquare}, leading to a minimization problem over the the hidden parameters $\bm{\xi}^{\text{H}}$ only, which can then be solved with a suitable optimization method. The variable projection method has been used for shallow (one hidden layer) neural networks in \cite{pereyra_2006}. A LS approach was also used in a greedy algorithm to generate adaptive basis elements by \citet{fokina2019growing}. 

In the approach presented here, instead of eliminating $\bm{\xi}^{\text{L}}$ through a LS solve, we alternate between the minimization of the two sets of parameters, $\bm{\xi}^{\text{L}}$ and $\bm{\xi}^{\text{H}}$, which is simpler to implement. In fact, with libraries such as Tensorflow \citep{tensorflow2015} and PyTorch \citep{pytorch}, one may automate extraction of the least squares problem (Equation \ref{eqn:lsq}) directly from the graphical representation of a neural network. Hence, algorithm \ref{alg:LSGD} may be easily implemented as a ``black-box'' layer on top of any architecture described by Equation \ref{eqn:NNbasis}.

\begin{figure}
    \centering
    \includegraphics[height=1.43in]
    {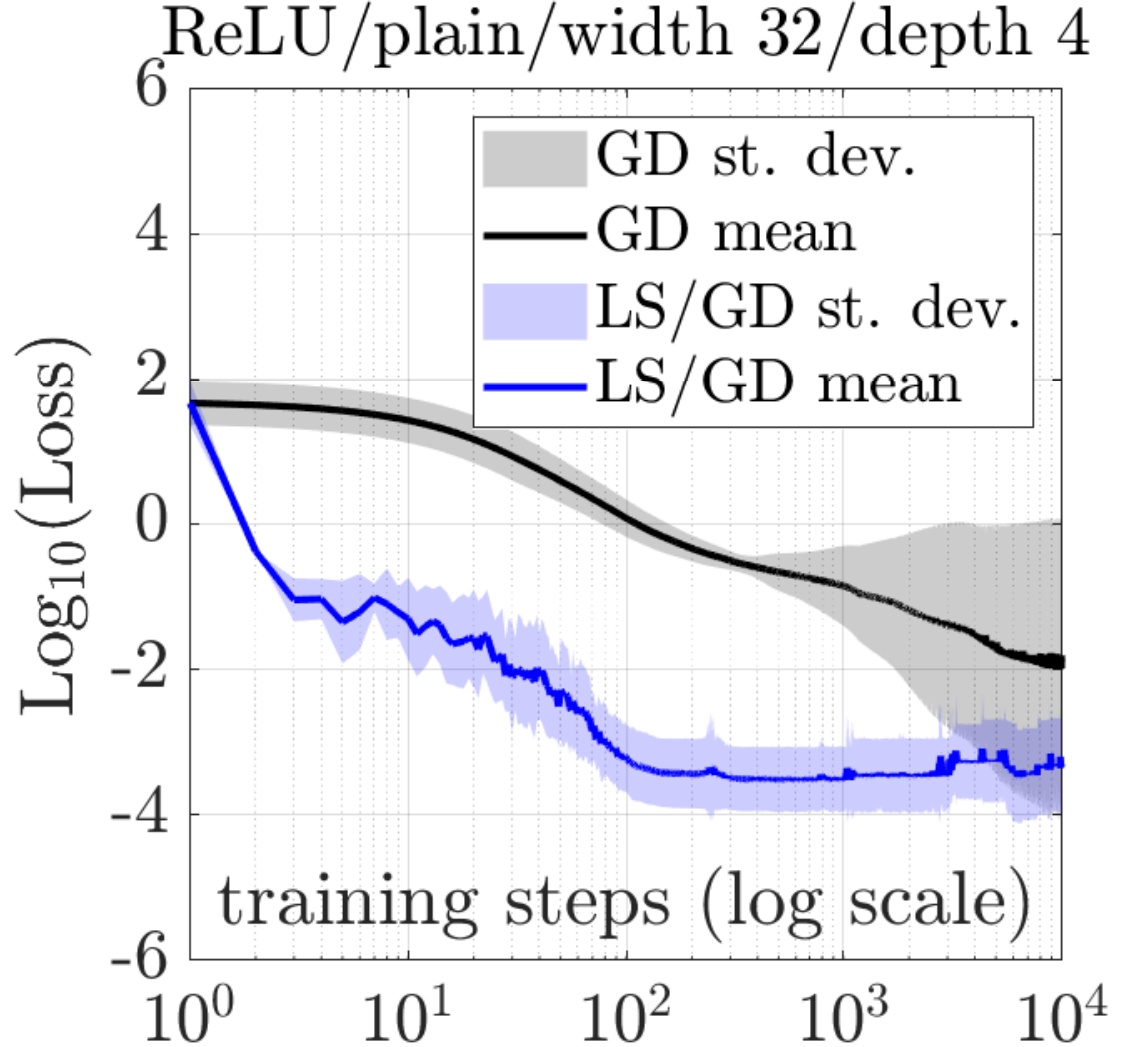}
    \includegraphics[height=1.43in]
    {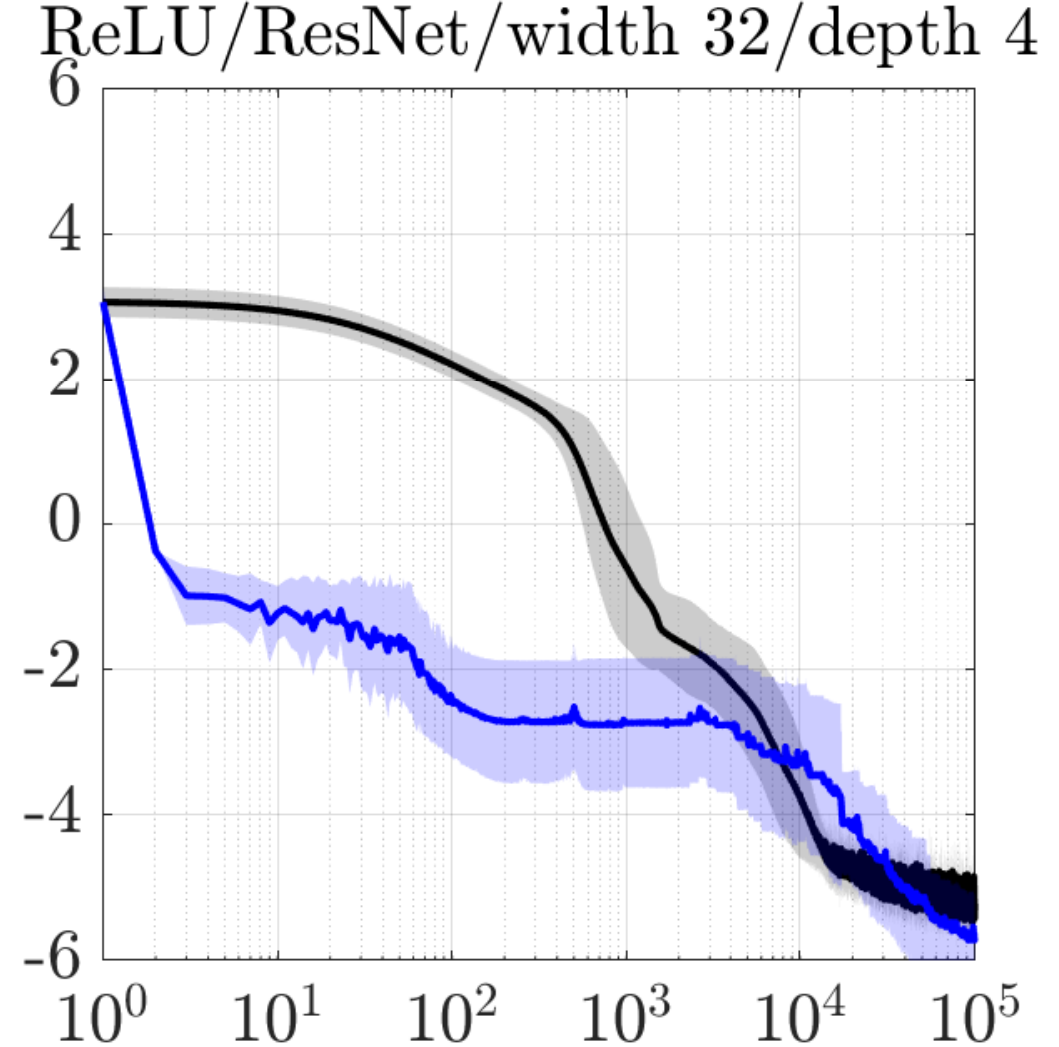}
    \includegraphics[height=1.43in]
    {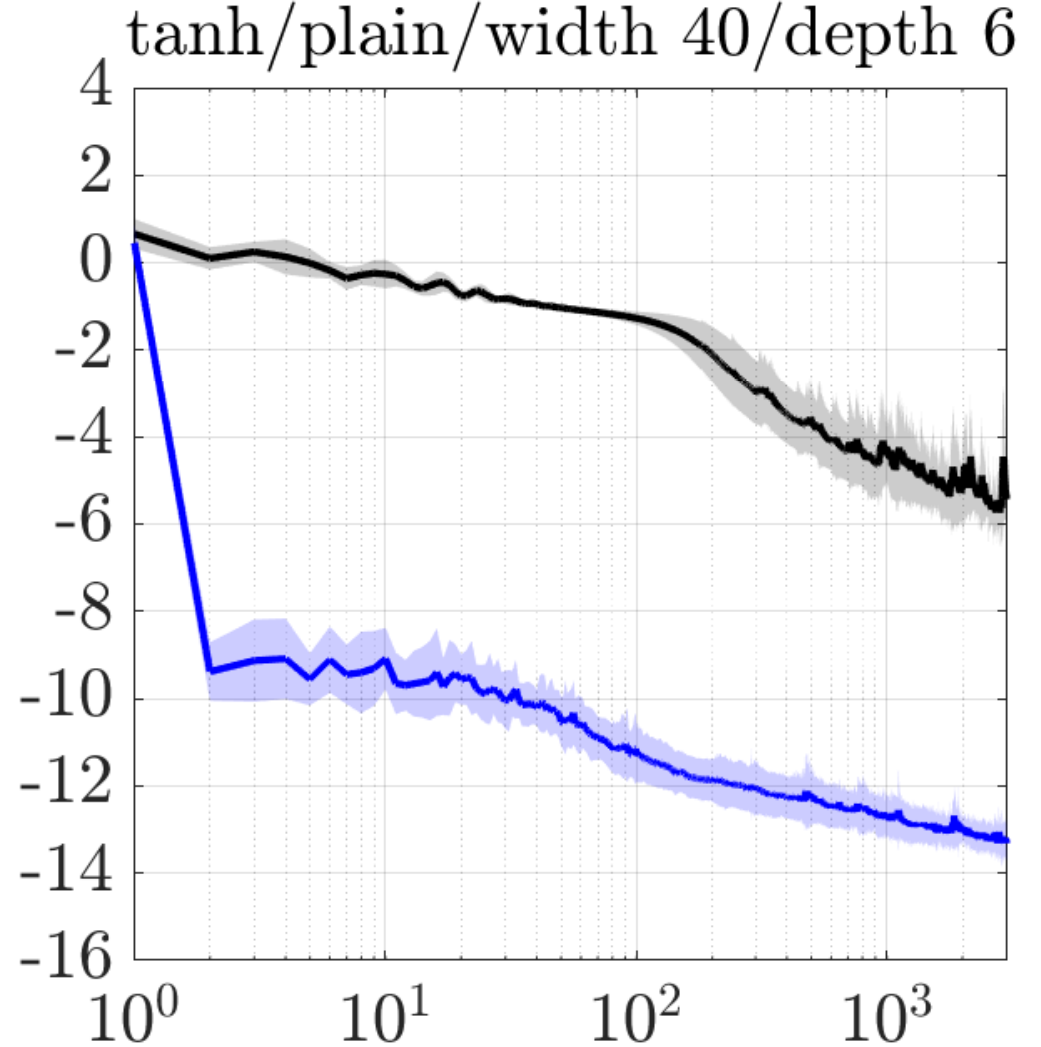}
    \includegraphics[height=1.43in]
    {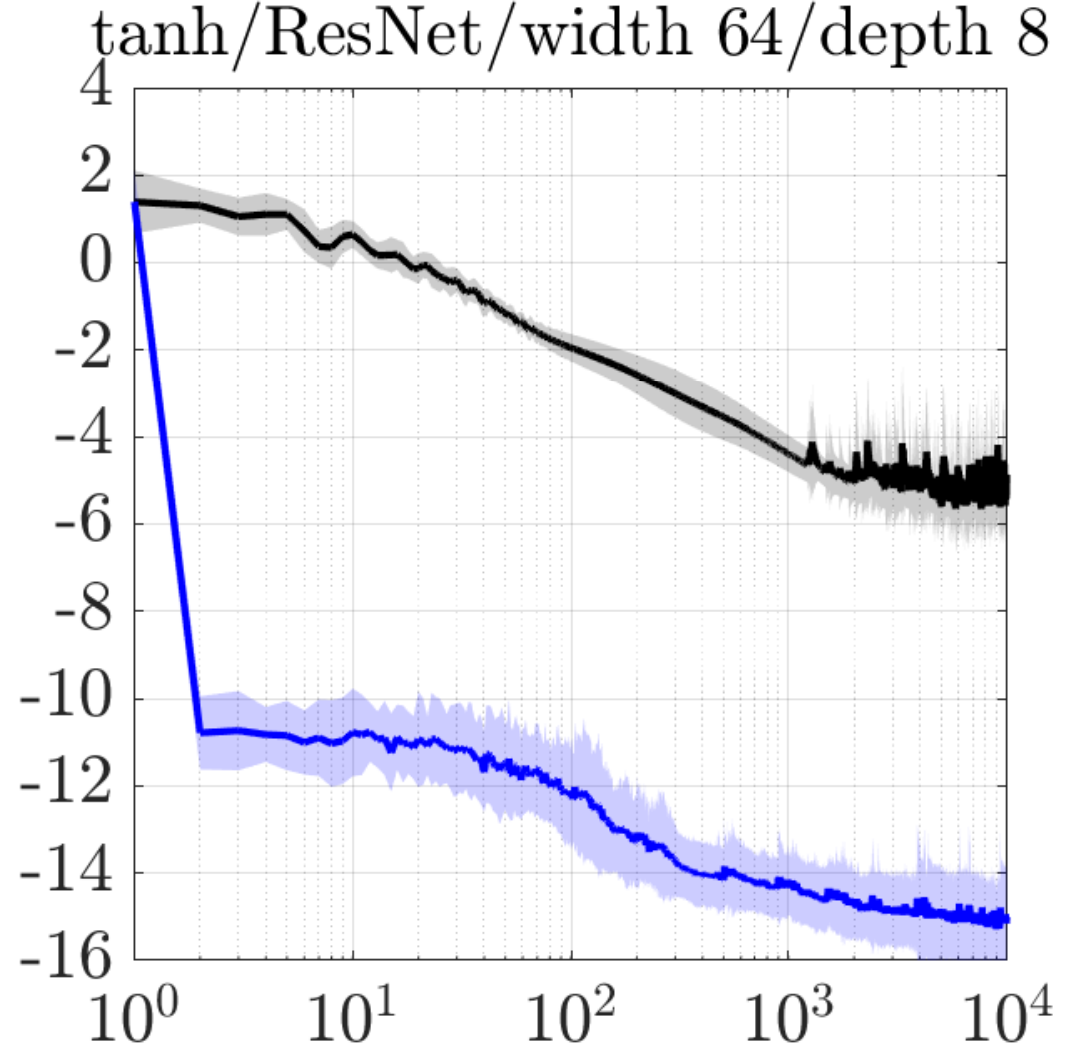}
\vspace{-4ex}
    \caption{Mean of $\log_{10}(\text{Loss})$ over 16 training runs $\pm$ one standard deviation of the same quantity. Training rate 0.0005 for GD and 0.005 for LSGD for plain network \emph{(left)} and ResNet \emph{(right)}.}
    \label{fig:lsgd_training}
\vspace{-4ex}
\end{figure}

We illustrate the advantages of LSGD training for approximating $\sin(2 \pi x)$ on $[0,1]$ using DNNs with ReLU and $\tanh$ activation in plain and ResNet architectures in Fig. \ref{fig:lsgd_training}. We use uniform He initialization and the Adam optimizer \citep{kingma2014adam} for the gradient descent steps; learning rates are tuned by hand to give stable training. We found that the LSGD optimizer performs best with a higher learning rate than that of GD -- roughly 10 times higher for ReLU networks, and 100 times higher for $\tanh$ networks. The results show that the loss in the LSGD method is typically several orders of magnitude lower than the loss in the GD algorithm after the same number of iterations. This is particularlly apparent for the $\tanh$ networks. However, we also included in Fig. \ref{fig:lsgd_training} a rare case in which the LSGD loss is momentarily overtaken by the pure GD loss to show that LSGD training and GD training do not admit a simple ``global'' comparison; for a further discussion of this as well as computational cost of LSGD, see Appendix \ref{lsgd_appendix}.

\section{The Box Initialization for deep ReLU networks} \label{sec:init}
\input{initialize.tex}

\section{Applications}
\subsection{One-dimensional regression}
\begin{figure}
    \centering
    \includegraphics[width=.6\textwidth]{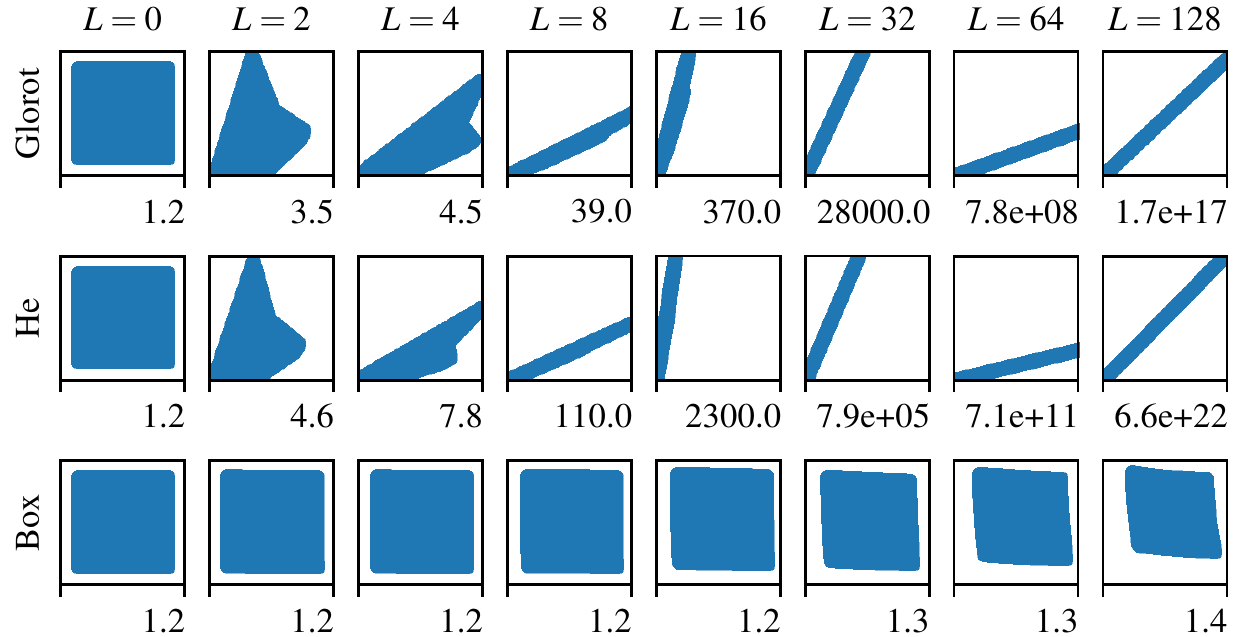}
    \hfill
    \includegraphics[width=0.39\textwidth]{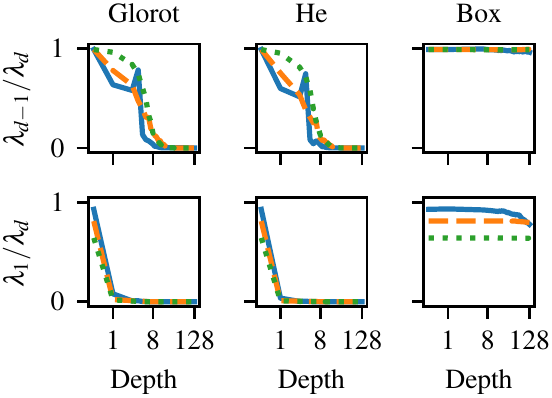}
    \vspace{-2ex}
    \caption{(\textbf{Left Subfigure:})
    Images of the unit square $[0,1]^2$ under $L$ initialized hidden layers of ResNets for Glorot (\textit{top}), He (\textit{center}) and Box (\textit {bottom}) initializations. 
    Values are presented on the square $\left[-0.2,H\right]^2$, where $H$ is denoted to the bottom-right of each image. Collapse to a line through the origin corresponds to linearly dependent basis functions (i.e., $\phi_1 = C \phi_2$)     (\textbf{Right subfigure:})
    Ratio of the second largest to largest eigenvalue (\emph{top}) and smallest to largest (\emph{bottom}) of the covariance of the image of samples from $\mathcal{U}[0,1]^d$.
    Results are shown for the ResNet architectures using dimensions $w=d=8$ \sym{line}{C0}, \edit{$32$} \sym{dash}{C1}, and \edit{$128$} \sym{dot}{C2}.
    } 
    \label{fig:resnet_image}
        
\end{figure}
In this section, we compare the behavior of the Glorot, He, and Box initializations for regression \edit{using ``very deep'' ReLU networks}. We first consider regression on the discontinuous function,

\begin{equation} \label{eqn:pp1d}
u_1(x) = \begin{cases} 
      x & 0\leq x< 0.5 \\
      1 - \frac{3}{4} x^2 & 0.5\leq x\leq 1
   \end{cases}.
\end{equation}

With a network width of \edit{$w=2$}, the three initializations, both Plain and ResNet architectures, and varying depths, we use the LSGD method to fit $u_1$.  Our results are shown in Figure~\ref{fig:padp} using an ensemble of initial random seeds for each initialization, architecture, and depth. Due to the narrow width of these network, only deep networks are capable of providing good approximations to $u_1$. However, we find that the Glorot and He initializations fail to find good fits to $u_1$
regardless of the architecture used.
\edit{While the Box initialization in the plain architecture also results in a poor fit, in the ResNet architecture it demonstrates a significant statistical improvement in training outcome.}

Our observations in Figure~\ref{fig:plain_image} and \ref{fig:resnet_image} suggest that the combination of initialization and architecture can lead to a starting condition in which the span of the basis functions is limited. The results in Figure~\ref{fig:padp}, related to problem~\ref{eqn:pp1d}, indicate that it is difficult to escape from this poor initial starting condition to a good fit. However, the Box initialization for the ResNet does not suffer from this lack of initial expressivity in the basis functions, and we are able to observe improvements increasing the depth of the network.

We next apply the Box initialization for ResNet to regress both $u_1$ and a smooth function, $u_2(x) = \sin\,2 \pi x$ for varying widths and depths. We observe first order convergence for the smooth function with respect to both width and depth, but only realize convergence with respect to width for the discontinuous functions (Figure \ref{fig:depthwidth}). 

\begin{figure}[htpb!]
    \centering
    \includegraphics[width=\textwidth]{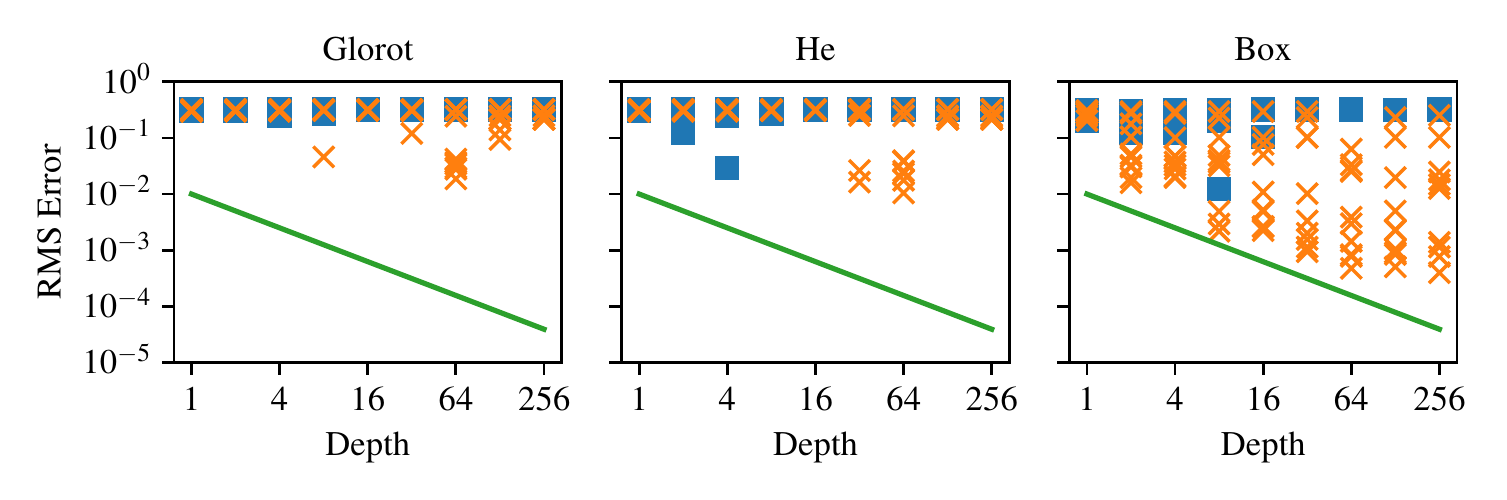}
    \vspace{-6ex}
    \caption{RMS Error for 1D piecewise polynomial polynomial regression (Equation~\ref{eqn:pp1d}) using the three initializations with Plain \sym{s}{C0} and ResNet \sym{x}{C1} architectures, respectively. Each symbol corresponds to the loss achieved using a different random seed for initialization. The green line \sym{line}{C2} indicates first order convergence with respect to depth. Setting: ReLU activation function, network width = 2, learning rate = 0.005.}
    \label{fig:padp}
\end{figure}

\begin{figure}[t!]
    \centering
    \includegraphics[width=0.49\textwidth]{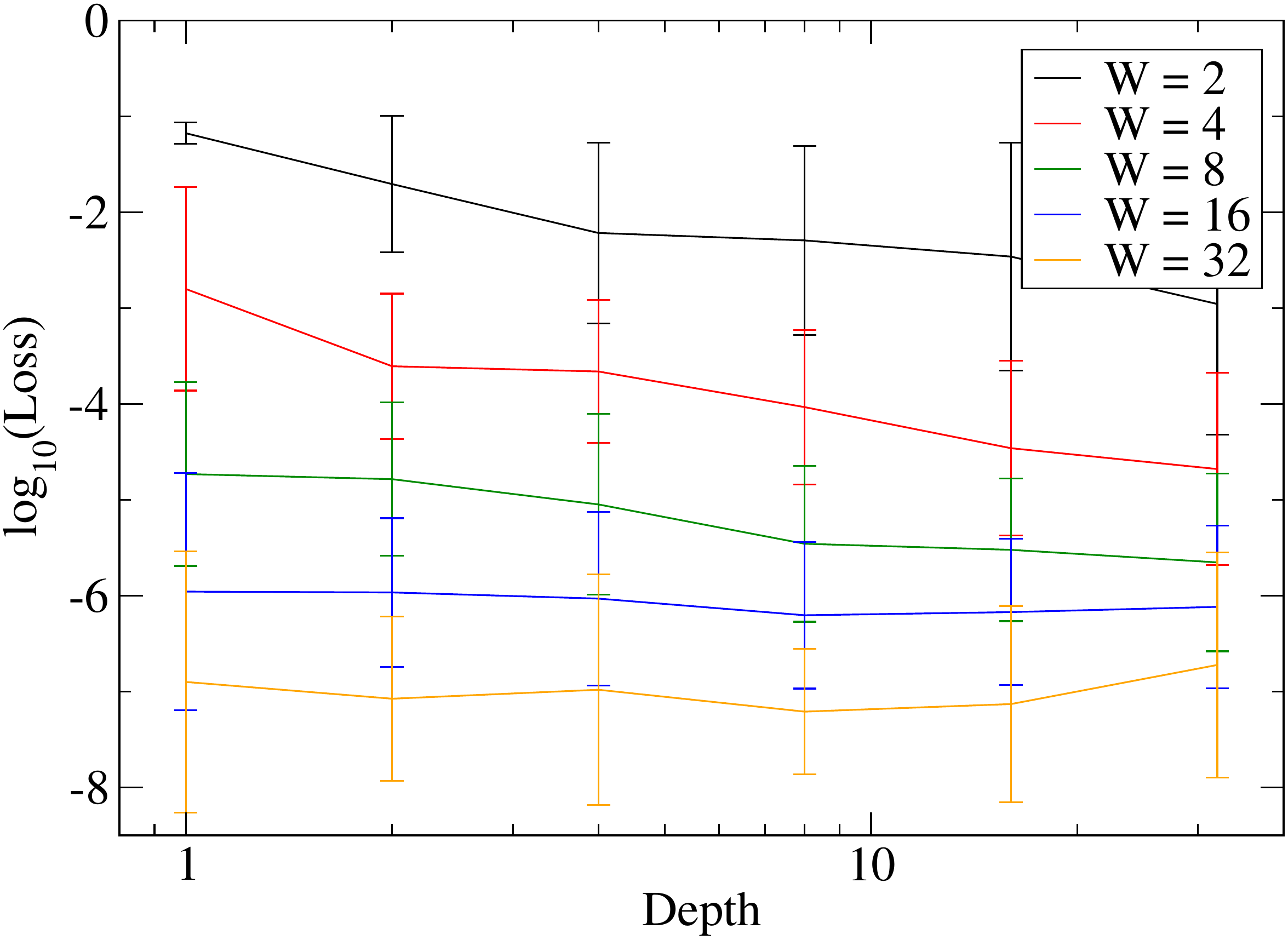}
    \includegraphics[width=0.49\textwidth]{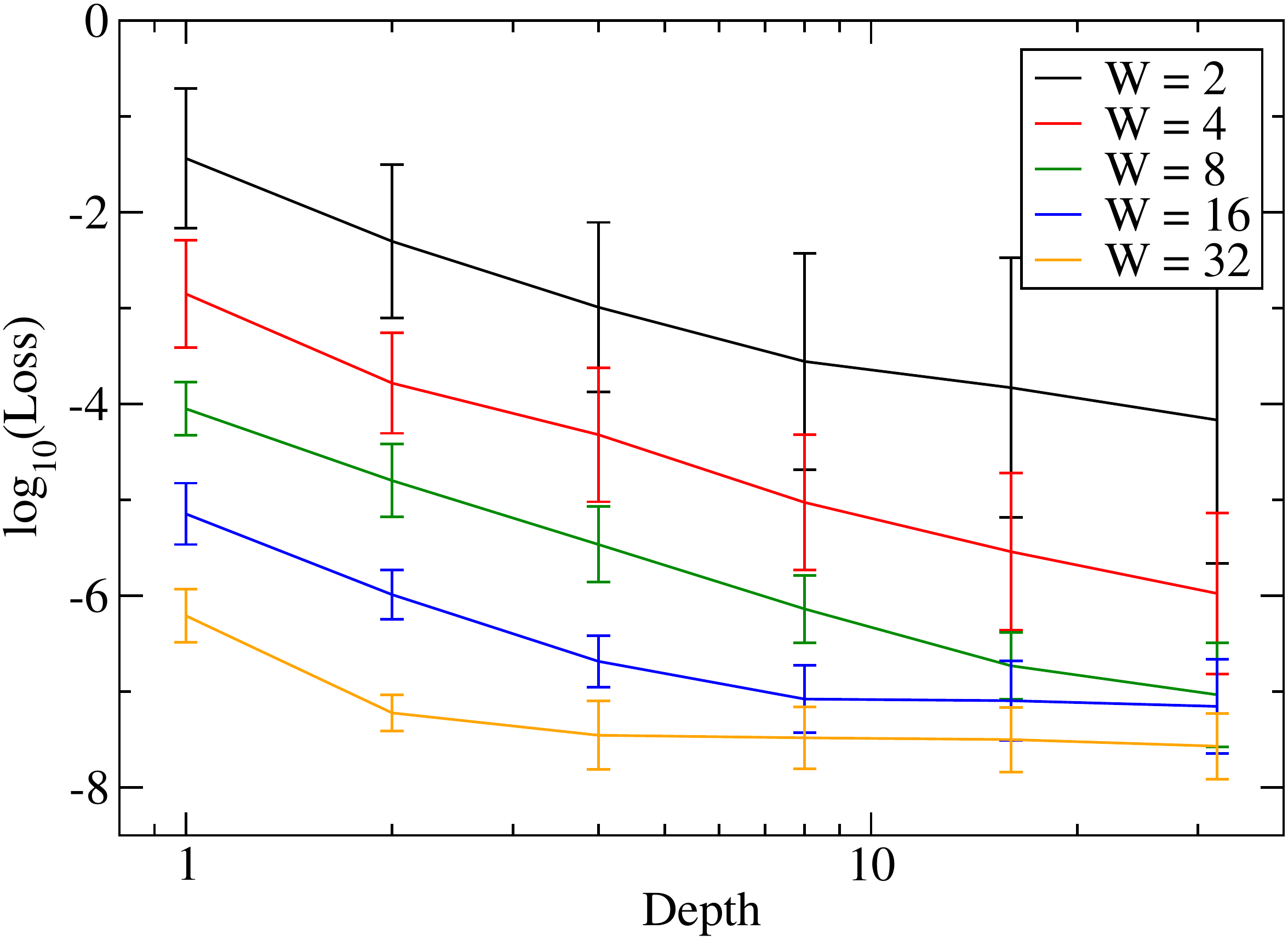}
    \vspace{-3ex}
    \caption{Convergence studies of regression with respect to width and depth on Equation~\ref{eqn:pp1d}  (\textit{left}) and  $u=\sin 2 \pi x$ (\textit{right}). Setting: ReLU activation function, ResNet architecture.}
    \label{fig:depthwidth}
\end{figure}

\subsection{Multi-function Regression}
\label{sec:multifunction}
The regression problem described above learns \edit{a basis that adapted to a} single function. In this section we modify the loss function so that the basis is defined to approximate a set of $N$ functions:
\begin{equation}
\label{multifunction_loss}
    \underset{\bm{\xi}^{\text{L}}}{\text{argmin }} 
\sum \limits_{n=1}^N   \left \| u_n -
 \sum_i \bm{\xi}_{n,i}^{\text{L}} \Phi_i(\bm{x},\bm{\xi}^{\text{H}})
\right \|_{\ell_2(\mathcal{X})}^2.
\end{equation}
Here the target functions are denoted $u_n$, and each has a corresponding set of linear coefficients $\bm{\xi}_{n,\cdot}^{\text{L}}$. The basis functions are defined by a single set of nonlinear weights $\bm{\xi}^{\text{H}}$, that define the output of a neural network as in single function regression described by Equation~\ref{eqn:std-regression}.

\begin{figure}[htpb!]
    \centering
    \includegraphics[width=0.45\textwidth]{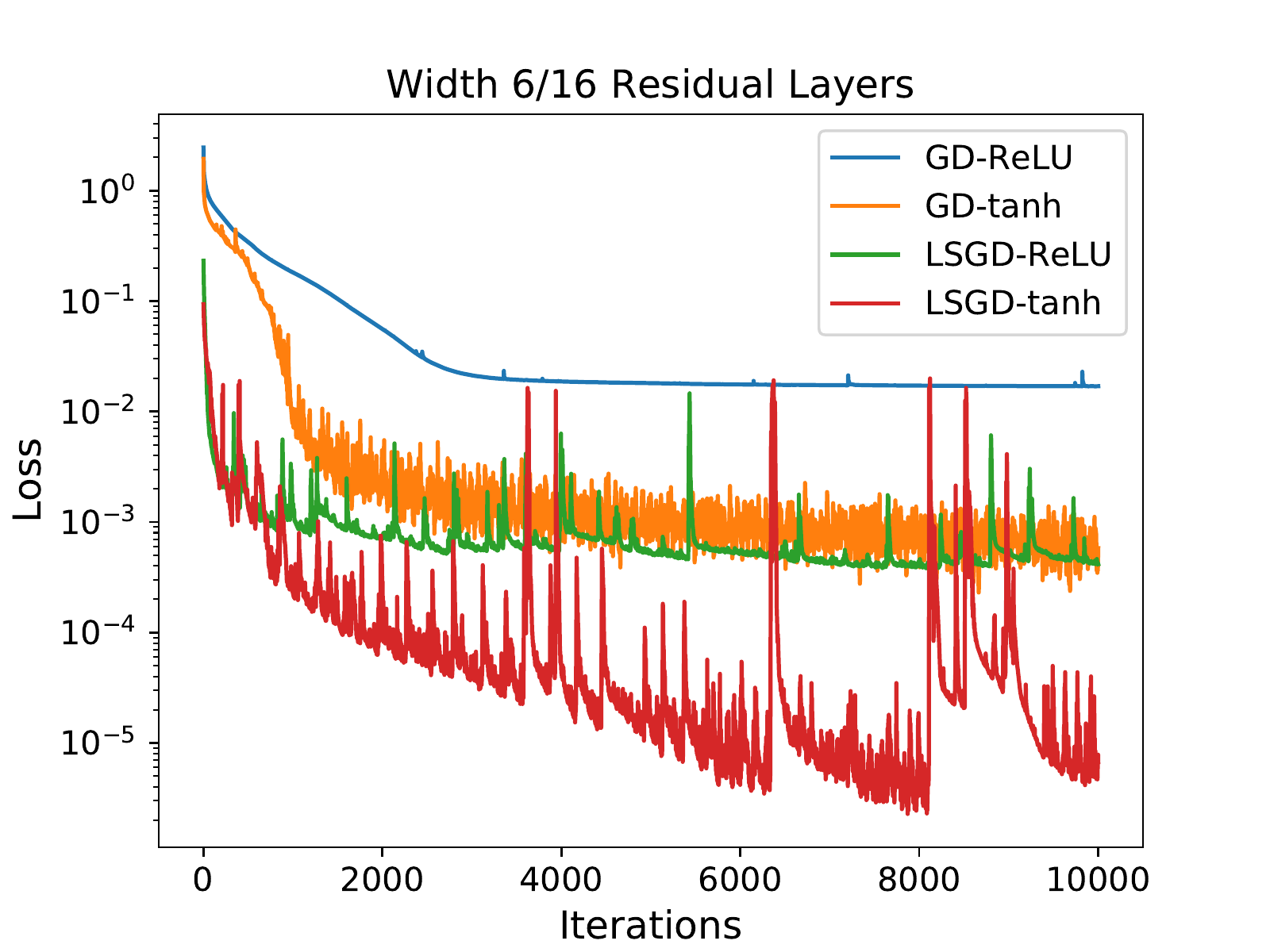}
    \includegraphics[width=0.45\textwidth]{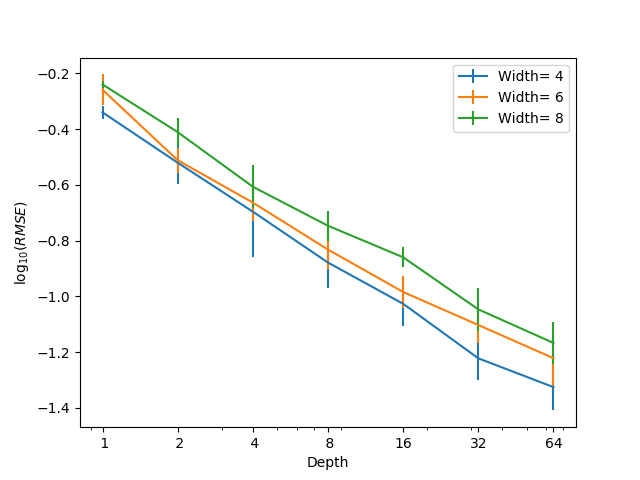}
    \vspace{-2ex}
    \caption{The left image shows the convergence of the loss of several multi-function regression problems with $6$ Legendre polynomials. Networks with $16$ residual layers and width $6$ using ReLU and $\tanh$ activation functions are used. Clearly, the LSGD training algorithm achieves smaller losses and converges more rapidly than GD. The right image shows the convergence as a function of depth of the RMS error. Note that as the width growths, so do the number of Legendre polynomials used in the objective function.
    }
    \label{fig:mo-loss-convergence}
\end{figure}

\edit{Our interest in multi-function regression lies in the fact that the adaptive basis representation of a DNN \eqref{eqn:NNbasis} exposes the problem \eqref{eqn:std-regression} as seeking a best $w$-term approximation to $u$ in the $\ell_2(\mathcal{X})$ norm. This is a form of nonlinear approximation that includes, e.g., wavelet and free-knot spline approximation \citep{cohen2009compressed, devore1998nonlinear}.
Here, the terms in the approximation $\Phi_i$, $i = 1, ..., w$ belong to the class of depth $L-1$ DNNs with input dimension $d$, output dimension 1, and nonlinear in the final layer; see Fig. \ref{fig:basisComparison}.
The multi-function regression problem \eqref{multifunction_loss} therefore appears closely related to nonlinear $w$-widths in approximation theory \citep{devore1989optimal}, and has potential for a reduced order modeling strategy \citep{hesthaven2016certified} in which subspaces are found as the span of $\{\Phi_i\}_{i=1}^{w}$ to minimize a loss function of the form \eqref{multifunction_loss} given a large collection of data $\{u_n\}$. While the benchmarks considered below are considerably simpler than such an application, this represents a promising direction for future work.}

A multi-regression problem is solved targeting the Legendre polynomials in $L^2([0,1])$, normalized to ensure equal weighting in the loss. The Legendre polynomials were chosen because of the range of structure in the set of polynomials. Note that the algorithm described above has not been modified to take advantage of their orthonormality. The left image in Fig.~\ref{fig:mo-loss-convergence} shows the convergence of networks with $16$ residual layers of width $6$ trained to match $6$ Legendre polynomials is studied (a one-to-one relationship between width and target functions). Here, the mean loss over $10$ repeated simulations is plotted as a function of iteration. The LSGD and GD training algorithms are compared. From the figure, LSGD reaches a smaller magnitude loss in fewer iterations than the equivalent network trained with GD. Furthermore, the usage of $\tanh$ leads to a smaller loss than with ReLU, thus better representing the set of Legendre polynomials. This is attributed to the broader support and greater smoothness of $\tanh$.

For the right image in Fig. \ref{fig:mo-loss-convergence} we use a ReLU ResNet with width $w$ to fit a space of Legendre polynomials of dimension $w$. For each realization we compute the error as the minimum over all iterations of the maximum RMS errors over the target polynomials and we then plot the mean RMS error over all the realizations. The learning rate for these simulations is set at $0.0005$. The image demonstrates that \edit{greater} accuracy is achieved 
as \edit{depth increases}.

\subsection{Physics-informed neural networks}\label{sec:pinns}
We consider now a physics-informed neural network (PINN) solution to the linear transport equation $\partial_t u(x,t) + a(x,t)\, \partial_x u(x,t) =0$ on the unit space-time domain $(x,t) \in [0,1]^2$, with initial condition $u(x,t=0)=u_0(x)$ and homogeneous Dirichlet boundary data $u(x=0,t)=0$. The loss function considered here is
\begin{align}
\begin{split}
    \mathcal{J} = \epsilon \mathcal{J}_1 + \mathcal{J}_2 +  \mathcal{J}_3, \qquad
    \mathcal{J}_1 = \frac{1}{N_1} \sum_{i \in \mathcal{X}_1} | \partial_t \mathcal{NN}_i + \partial_x a(x,t) \mathcal{NN}_i |^2, \\
        \mathcal{J}_2 = \frac{1}{N_2} \sum_{i \in \mathcal{X}_2} | \mathcal{NN}_i(x,0) - u_0 |^2, \qquad
            \mathcal{J}_3 = \frac{1}{N_3} \sum_{i \in \mathcal{X}_3} | \mathcal{NN}_i (0,t)|^2
\end{split}
\label{eqn:pinnLoss}
\end{align}
where $\mathcal{X}_1,\mathcal{X}_2$ and $\mathcal{X}_3$ are Cartesian point clouds with spacing $\Delta x$ on the interior, left and bottom boundaries, respectively. We note that the loss function is typically further augmented with a term to match given data (see e.g. \cite{raissi2019physics}), and PINNs thus amount to regularizing traditional regression with the least-squares solution of a collocation scheme using the neural network as basis. 
For all results we will use ResNets and consider as initial condition a tent function $u_0 \in C_0$. 

It is an open question how to choose the parameter $\epsilon$ scaling the first term of the $\mathcal{J}$ so that the three competing loss functions have the same magnitude under refinement - in the literature this penalty parameter is tuned to a given architecture to demonstrate good agreement, but preventing a formal convergence study. Traditionally in a FEM penalty method, one would scale by a mesh diameter $h$ so that each term in Equation \ref{eqn:pinnLoss} has consistent units, and comparable magnitude. In the current context, the adaptive basis has no inherent lengthscale, as the gradient of the basis may grow arbitrarily large as the hidden weights evolve and cut planes may approach each other. 

\begin{figure}[htpb!]
    \centering
    \includegraphics[width=\textwidth]{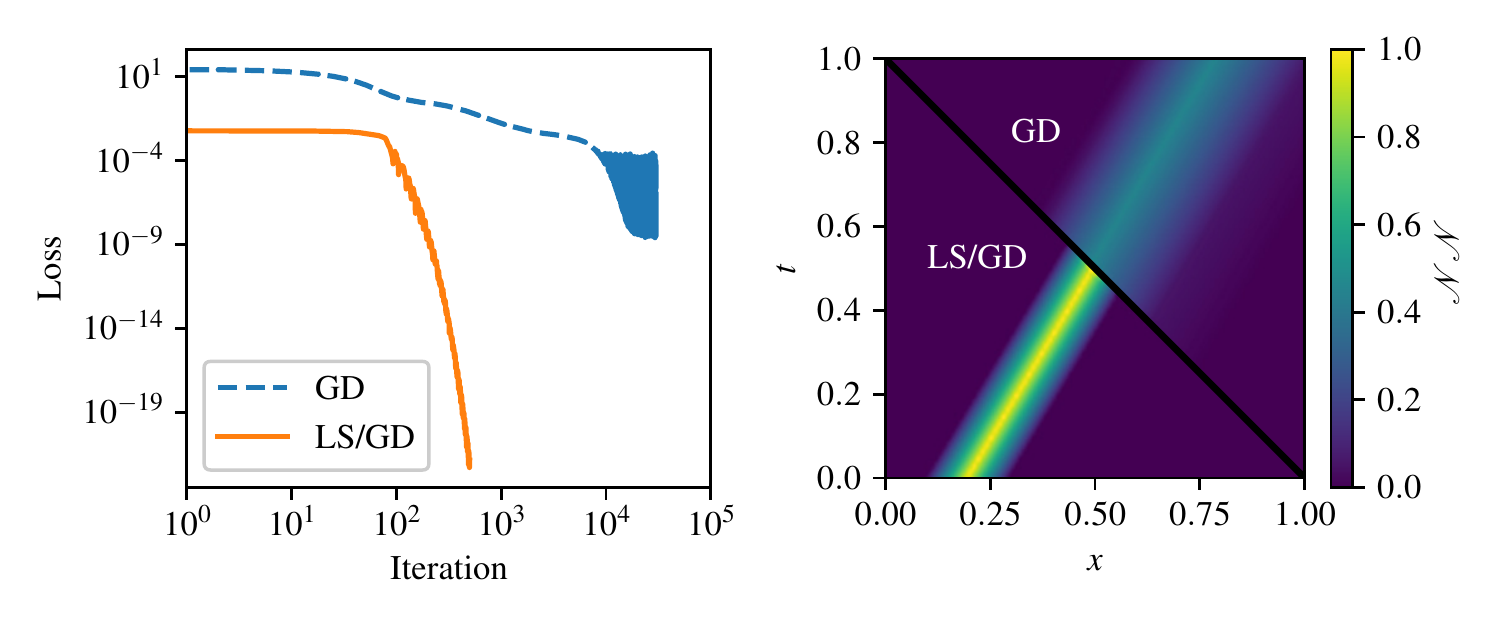}\\
    \vspace{-3ex}
    \caption{\textit{Left:} PINNs solution for transport equation with constant velocity. Loss evolution over training for GD and LSGD. \textit{Right:} Solution after 5000 iterations for GD and 500 iterations for LSGD. Setting: Box initialization, ReLU activation function, network width = 32, depth = 1, learning rate = 0.005. }
    \label{fig:constadv}
\end{figure}

\begin{figure}[htpb!]
    \centering
    \includegraphics[width=0.6\textwidth]{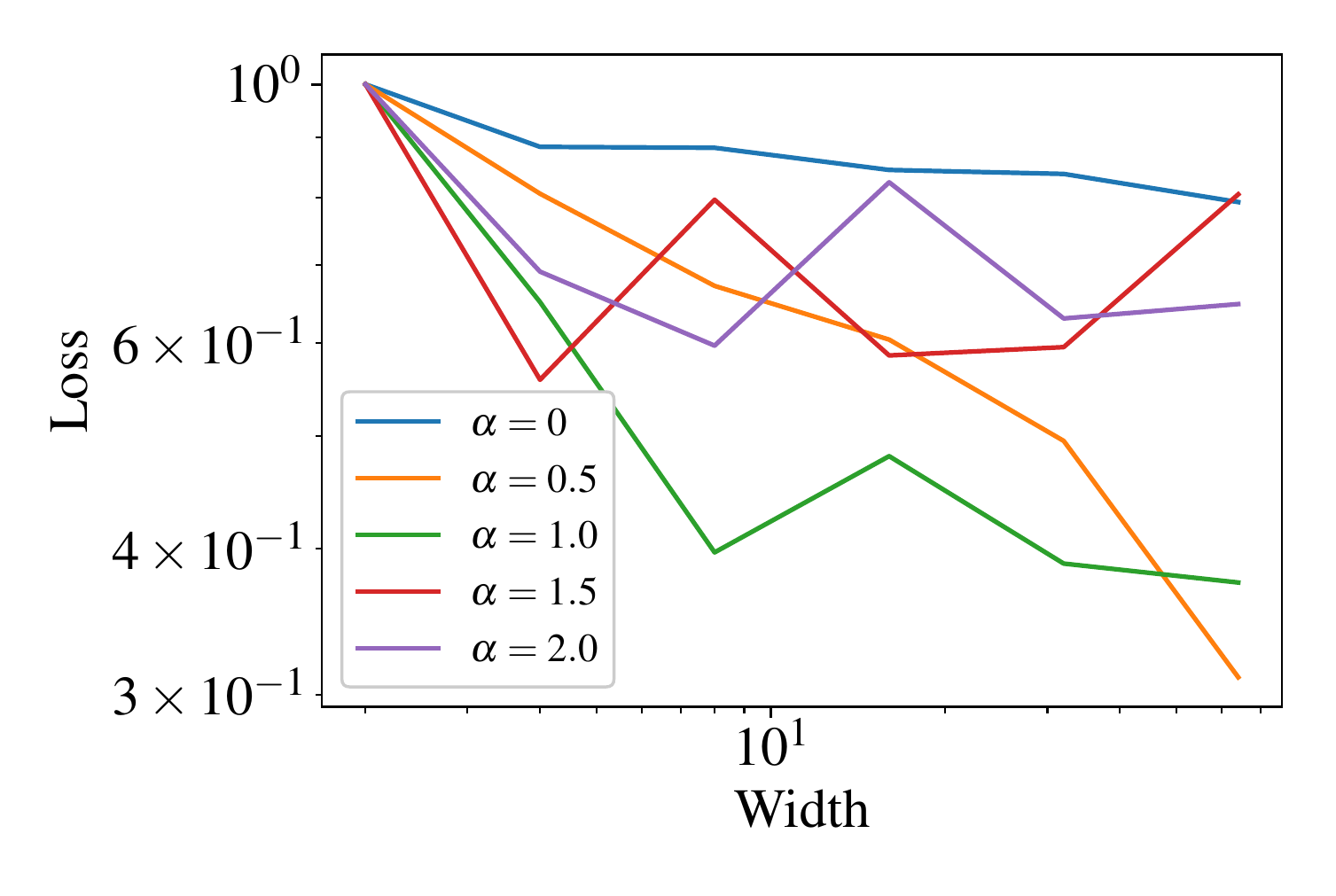}\\
    \includegraphics[width=0.4\textwidth]{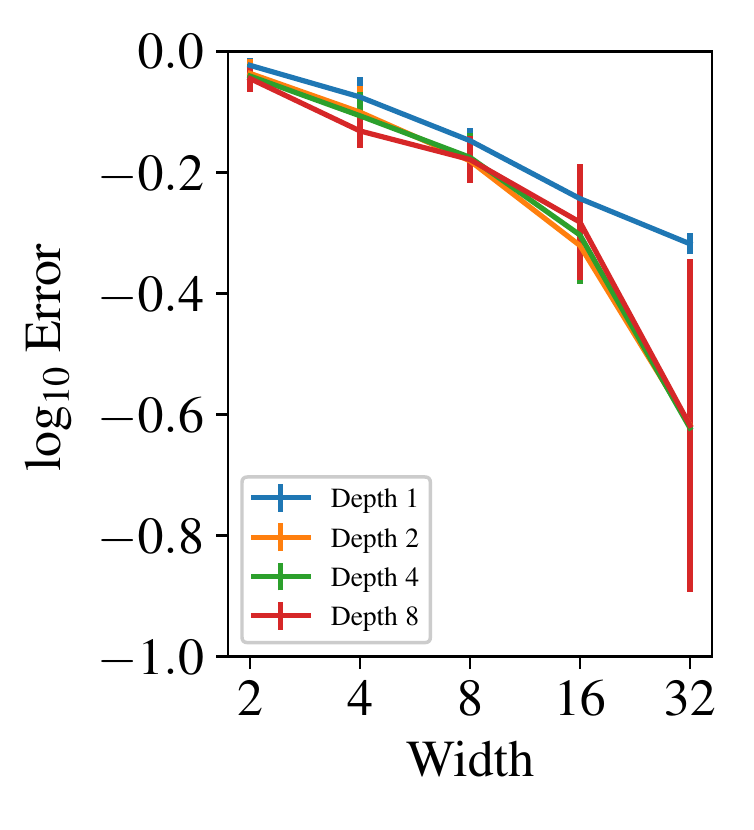}
    \includegraphics[width=0.4\textwidth]{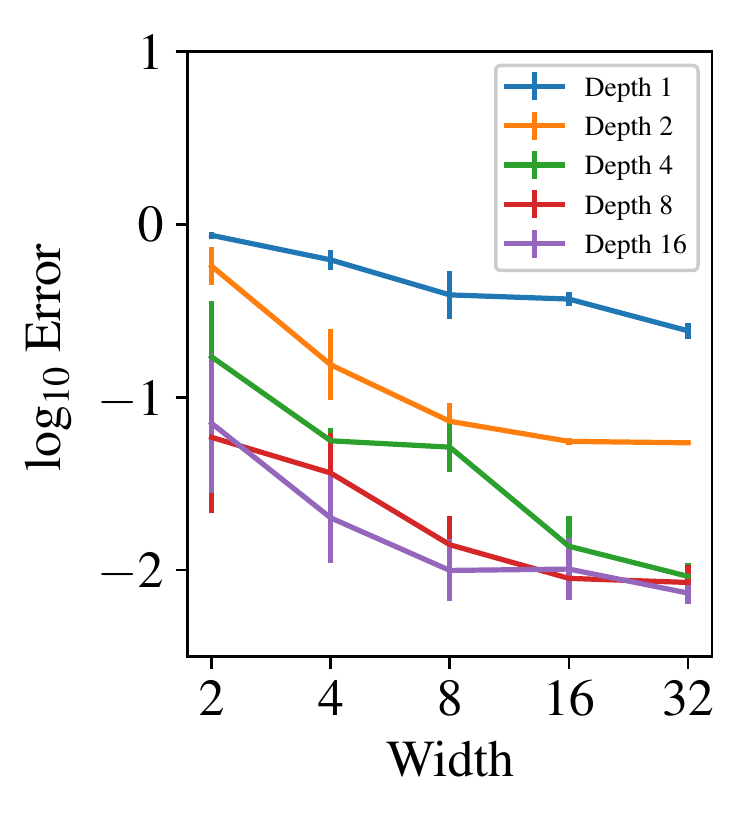}\\
    \includegraphics[width=0.95\textwidth]{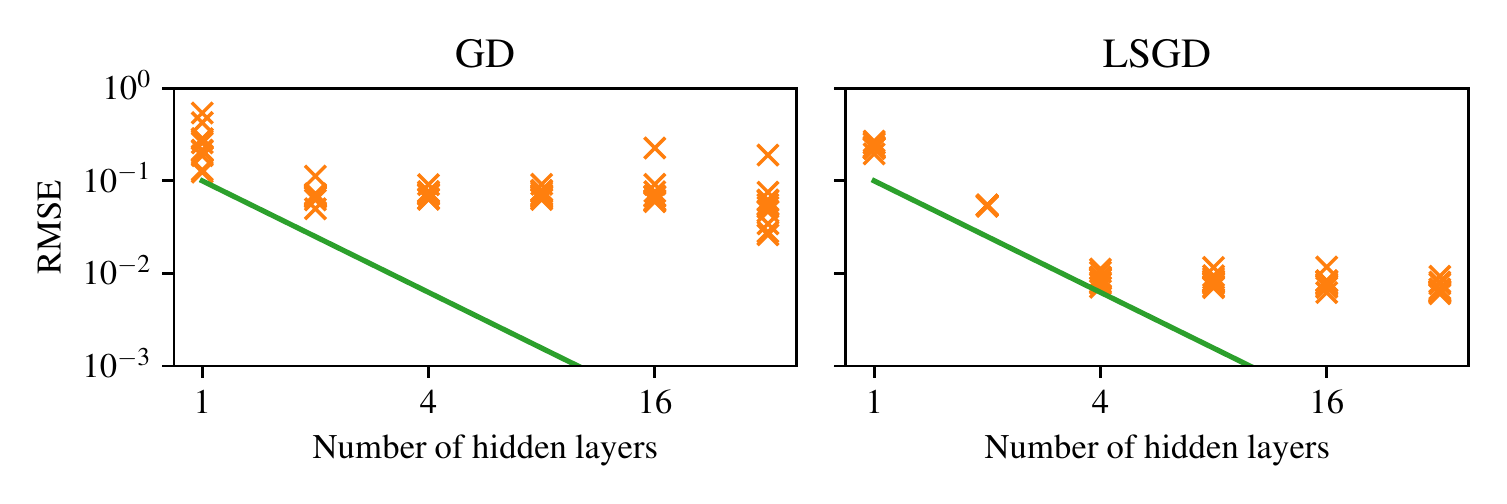}
    \caption{RMS error for ResNet PINNs with Box initialization for the nonconstant velocity case.
     \textit{Top:} Convergence of ReLU-PINNs solutions with respect to penalty scaling $\epsilon \sim W^{-\alpha}$.
     \textit{Middle:} Convergence of ReLU (\textit{left}, learning rate 0.001) and tanh (\textit{right}, learning rate 0.01) using $\alpha = -\frac{1}{2}$.
    \textit{Bottom:} 
    Comparison of between GD (left) and LSGD (right) training for tanh activation functions, learing rate 0.01, width 32 and 5000 epochs. The x's indicate the errors  as a function of the number of layers for different realizations of the Box initialization. The line indicates second order convergence with respect to depth.  
}
    \label{fig:nonconstadv}
\end{figure}

We first consider in Figure~\ref{fig:constadv} the case of constant velocity, $a(x,t) = 1$, with corresponding analytic solution $u(x,t) = u_0(x-t)$, and use a shallow one-layer ReLU network. For this case, the exact solution is in the range of the network for width $\geq 3$, and at this point $\mathcal{J}_1=\mathcal{J}_2=\mathcal{J}_3=0$, rendering the choice of $\epsilon$ unimportant (we set $\epsilon = 1$). In this case we observe similar trends to the previous sections; the proposed LSGD training strategy converges to $10^{-15}$ in double precision with orders of magnitude fewer iterations than GD. From the evolution of the cut planes during training (see Appendix \ref{pinnsSnapshots}), it is clear that the basis is adapting to the characteristics of the PDE.

We next consider nonconstant velocity, $a(x,t) = x$, with corresponding analytic solution $u(x,t) = u_0(x \exp(-t))$ (Figure \ref{fig:nonconstadv}). In this case we must fix $\epsilon$ independent of the neural network size to realize convergence, and we hypothesize $\epsilon = W^{-\alpha}$. Solutions for $\alpha \in \left\{0,\frac12,1,\frac32,2\right\}$ reveal $O(W^\frac12)$ convergence for $\alpha = \frac12$. Following the FEM interpretation of shallow networks (\cite{he2018relu}), we interpret $h \sim N^{-\frac1d}$, and selecting $\epsilon=W^{-\frac12}$ corresponds to non-dimensionalizing the loss, allowing a realizatiion of first-order convergence with respect to $h$. To consider the effect of depth we repeat the previous experiment for increasingly deep ReLU and tanh networks. Finally, to gauge the effectiveness of our training strategy, we compare using GD only vs LSGD.

While a thorough study of PINNs is beyond the scope of this paper, we conclude that the combination of LSGD, Box initialization, and choice of $\epsilon$ provides substantial performance gains relative to traditional GD, and we conclude that depth plays an important role in the convergence of PINNs.

\section{Conclusions}
Motivated by recent theoretical advances in the approximation theory of DNNs, this work takes an adaptive basis viewpoint of neural network training and initialization. This perspective naturally leads to a hybrid least squares/gradient descent training algorithm. We demonstrate that this approach leads to accelerated training for regression, multi-function regression and physics-informed neural networks, in the context of both ReLU and $\tanh$ activation functions. In a novel development, we proposed a new ``box initialization'' procedure inspired by the basis viewpoint that dramatically enhances the training of deep ReLU networks. As part of this we analyzed a potential failure mode for certain initializations that leads to a highly linearly dependent initial basis and demonstrated this failure for the Glorot and He initializations that are commonly used to initialize ReLU networks. For ResNets, we showed how the box initialization leads to a significantly improved basis, ultimately leading to more efficient training than the He initialization. Finally, using the combination of both Box initialization and LSGD training, we demonstrate in \edit{several} scenarios the ability for neural networks to achieve relatively robust convergence as a function of both width and depth, for both \edit{single- and multi-function} regression problems and PDE applications using physics-informed neural networks.

That machine learning algorithms can be understood as providing an underlying adaptive basis from data is a viewpoint that permeates many areas of deep learning \citep{murphy2012machine, he2018relu, fokina2019growing, wang2019stochastic}\edit{, one that is amenable to numerical analysis and admits comparison to well-established methods in scientific computing.}
The techniques developed here, in addition to improving the training of neural networks, demonstrate how an adaptive basis perspective can be used to attack critical issues hindering the robustness of machine learning. Taking a numerical analysis viewpoint has shed new light on the issues confronting neural network training and has provided intuition regarding the use of physics-informed neural networks to solve PDEs. We believe that additional advances are possible when considering the numerical implications of \edit{strategies} in \edit{deep learning}. Our work aims to strengthen the numerical properties of existing \edit{DNN} approaches and also provide a mathematical foundation in response to the need suggested in~\cite{baker2019workshop} to obtain rigorous results for use in scientific machine learning.

\acks{Sandia National Laboratories is a multimission laboratory managed and operated by National Technology and Engineering Solutions of Sandia, LLC, a wholly owned subsidiary of Honeywell International, Inc., for the U.S. Department of Energy’s National Nuclear Security Administration under contract DE-NA0003525. This paper describes objective technical results and analysis. Any subjective views or opinions that might be expressed in the paper do not necessarily represent the views of the U.S. Department of Energy or the United States Government. SAND Number: SAND2019-14904 C.}

\newpage
\bibliography{refs}

\newpage
\appendix
\section{Properties and Performance of LSGD training}
\input{lsgd_appendix}
\newpage
\section{Box Initialization Algorithms \& Details}
\input{algorithms_appendix.tex}

\newpage
\section{Basis Function Plots}
\input{basis_function_appendix.tex}
\newpage
\input{pinns_snapshots.tex}

\end{document}

%% file: initialize.tex
The first step in Algorithm \ref{alg:LSGD} is to initialize the hidden layer parameters. 
An initialization resulting in a well-conditioned, basis that is linearly independent in $\ell_2(\mathcal{X})$ will provide a richer approximation space for the least squares problem and give the gradient descent optimizer several ``active'' basis functions to tune. In contrast, an initialization leading to poorly-conditioned, linearly dependent basis functions -- such as a basis functions with support disjoint from the data -- will yield a less expressive basis in which a local variation of the hidden parameters may not improve the loss.

\subsection{Plain Neural Networks}
\label{plain_init_subsec}
Analyses of the representation power of ReLU networks have shed light on the role played by the biases for representing continuous piecewise linear (CPWL) functions \citep{arora2016understanding, hanin2017universal, hanin2017approximating, he2018relu}. 
For example, for CPWL functions of one variable, \citet{he2018relu} identified their single layer ReLU network representations $\sum \lambda_i \text{ReLU}(x-\beta_i)$ with nodal finite element representations, with the nodes given by $\beta_i$. 
In higher dimensions, the cut planes (See Figure \ref{fig:initDiagram}) defined by the bias vectors of single layer ReLU networks correspond to the facets of a CPWL finite element mesh. This implies that to obtain a ``feature-rich'' initial basis, assuming the data input is normalized to $[0,1]^d$, one should scatter the cut planes of the ReLU functions over $[0,1]^d$ randomly.  

Loosely speaking, if the above initialization results in hidden layer with  ``feature-rich'' output, it is reasonable to speculate that composing two such layers has a good chance to also result in a ``feature-rich'' output, provided the first layer maps \emph{into} the domain of the second layer and is as close as possible to being \emph{onto}. 
The idea behind the ``box initialization'' for plain networks is to normalize the \emph{output} of each layer to $[0,1]^d$. The goal is to apply the above initialization inductively for each hidden layer and prevent ``blow-up'' of the initial basis for deeper networks. \edit{In the remainder of this section, we consider neural network architectures in which the width of the hidden layers is a constant $w$ throughout the network. This simplifies the analysis, although the algorithm can be considered for networks with variable hidden layer width $w_l$, $l = 1, ..., L$; see Appendix \ref{algorithm_appendix}.}

\begin{figure}
  \begin{minipage}[c]{0.25\textwidth}
    \includegraphics[width=\textwidth]{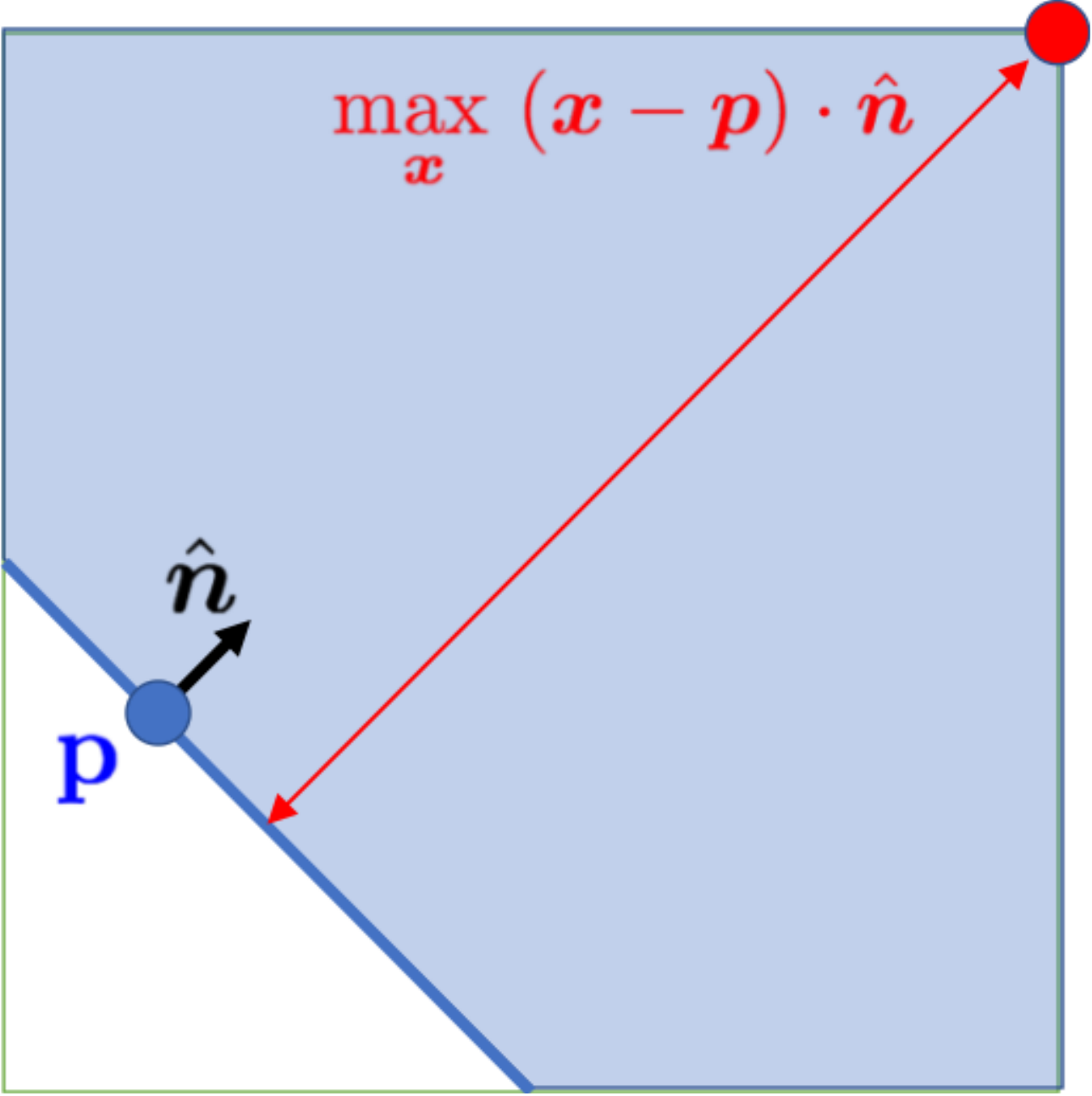}
    \end{minipage}\hfill
      \begin{minipage}[c]{0.66\textwidth}
    \caption{Notation used in the ``box initialization'' of each node. A random point $\bm{p}$ with random orientation $\hat{\bm{n}}$ is used to define a ReLU function of form $\sigma(k (\bm{x} - \bm{p})\cdot\hat{\bm{n}})$. Using Lemma \ref{maurosTrick}, one may choose the slope of the ReLU $\alpha$ to impose an upper bound on the output of each layer. We refer to the hyperplane normal to $\hat{\bm{n}}$, where the ReLU ``switches on'', as the \emph{cut plane}.}
    \label{fig:initDiagram}
  \end{minipage}
\end{figure}
Referring to Fig.~\ref{fig:initDiagram}, the procedure is for each output row ($1\ldots i \ldots w$) of the layer:
\begin{enumerate}[noitemsep,nolistsep]
\item
\edit{Select $\bm{p} \in [0,1]^w$ at random.} 
\item
Select a normal $\bm{n}$ at $\bm{p}$ with random direction.
\item
Choose a scaling $k$ such that 
\begin{equation}
\underset{{\bm{x}\in}[0,1]^w}{\text{max }}\sigma(k(\bm{x}-\bm{p})\cdot \bm{n}) = 1. 
\label{eqn:plain-init-assumption}
\end{equation}
\item
Row $\bm{w}_i$ of $\bm{W}^{\bm{\xi}}$ and $\bm{b}^{\bm{\xi}}$ are selected
as $b_i = k\bm{p}\cdot\bm{n}$ and $\bm{w}_i = k\bm{n}^T$.
\end{enumerate}
\edit{To initialize the first hidden layer, replace $w$ by the input dimension $d$ in steps 1 and 3 above.}
A full description of this initialization and an efficient way to calculate the $k$ may be found in Algorithm~\ref{alg:dense-init} in Appendix \ref{algorithm_appendix}.
With the layer initialized as above, consider feeding a box $[0,1]^w$ as input into a given layer.
For a plain neural network, the output $\bm{x}_{l+1}$ of layer $l$ is given by 
\begin{equation}
\bm{x}_{l+1} = \sigma({\bm{W}_l} \bm{x}_{l} + \bm{b}_{l}).
\end{equation}
Then, we have for every component $i \in \{ 1, 2, \hdots, w\}$,
\begin{equation}\label{eqn:plain-min-max}
\underset{\bm{x}^l \in [0,1]^w}{\text{min }} (\bm{x}_{l+1})_i = 0; \quad
\underset{\bm{x}^l \in [0,1]^w}{\text{max }} (\bm{x}_{l+1})_i = 1.
\end{equation}
Equation~\ref{eqn:plain-min-max} implies that layer $l$ maps $[0,1]^w$ into $[0,1]^w$. Moreover, ensuring the extrema are achieved on $[0,1]^w$ guarantees its image intersects each side of the hypercube at least at a point.
This does not imply however, that each layer map from $[0,1]^w$ into $[0,1]^w$ is onto. Nor, as we will see, that the the composition of two layer maps will have guaranteed intersections with the boundary. Assuming the input into the first hidden layer is contained in $[0,1]^w$, then box initializion ensures that the hidden layers initially map
\begin{equation}
[0,1]^d 
\xrightarrow[]{\text{into}}
[0,1]^w 
\xrightarrow[]{\text{into}}
\left[0,1\right]^w 
\xrightarrow[]{\text{into}}
\left[0,1\right]^w 
\xrightarrow[]{\text{into}}
\hdots
\xrightarrow[]{\text{into}}
\left[0,1\right]^w.
\end{equation}

In Figure~\ref{fig:relu_shallow_initialization} we compare the the box initialization for a plain ReLU network with width $w=32$ against the He (see \cite{he2015delving}) initialization for approximating $\sin(2\pi x)$ on $[0,1]$. We average over 16 independent training runs. The box initialized basis is significantly richer for up to 8 layers, yielding a loss 2-4 orders of magnitude lower than that of the initialized He basis after the first least squares step. This is borne out by the plots of the initialized basis in Fig. \ref{fig:basis_plain}. The loss after $10^4$ LSGD steps is also lower by 2 orders of magnitude. Despite this promising improvement over He initialization, the box-initialized ReLU network with 16 layers fails to train, and plotting the basis function reveals they are constant over the input to the network;
see Appendix \ref{basis_appendix}. 
\begin{figure}
    \centering
    \includegraphics[height=1.8in]{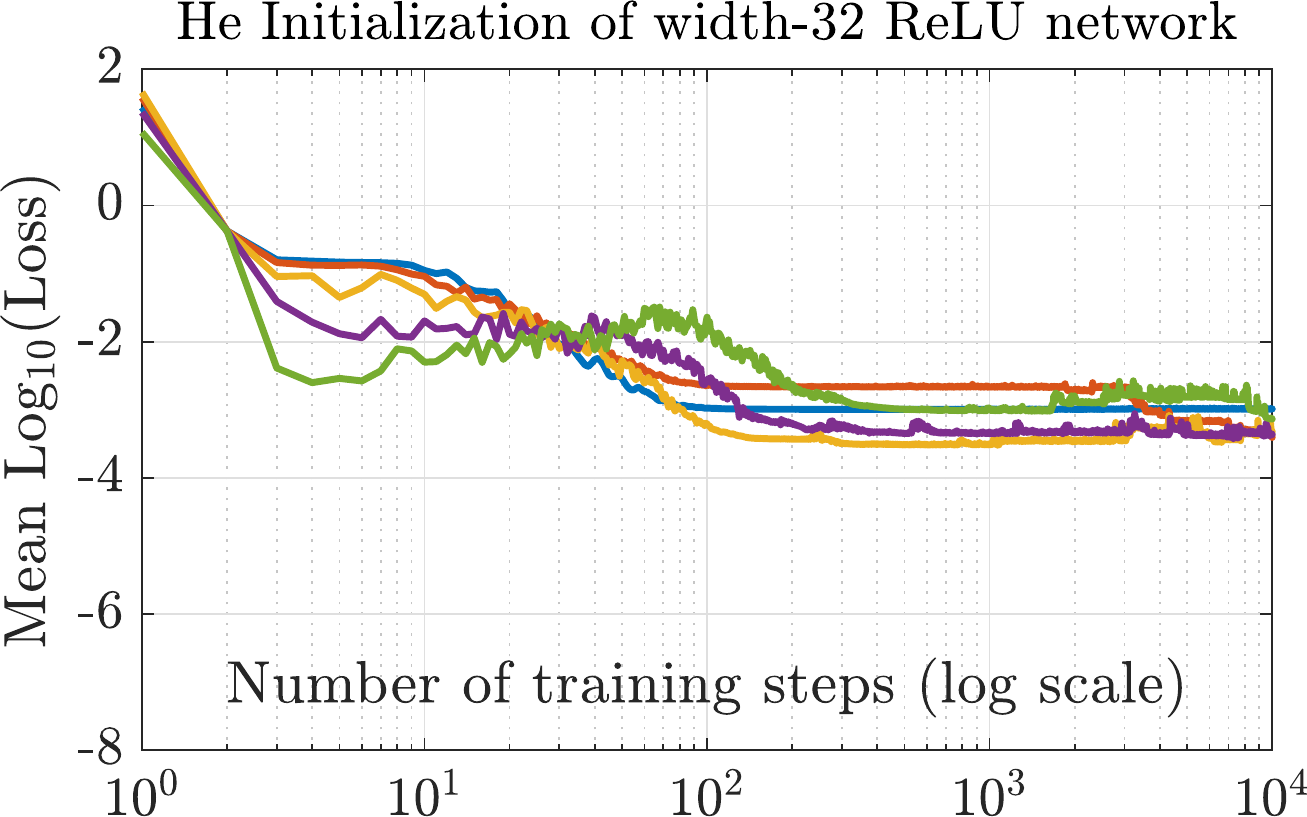}
    \includegraphics[height=1.8in]{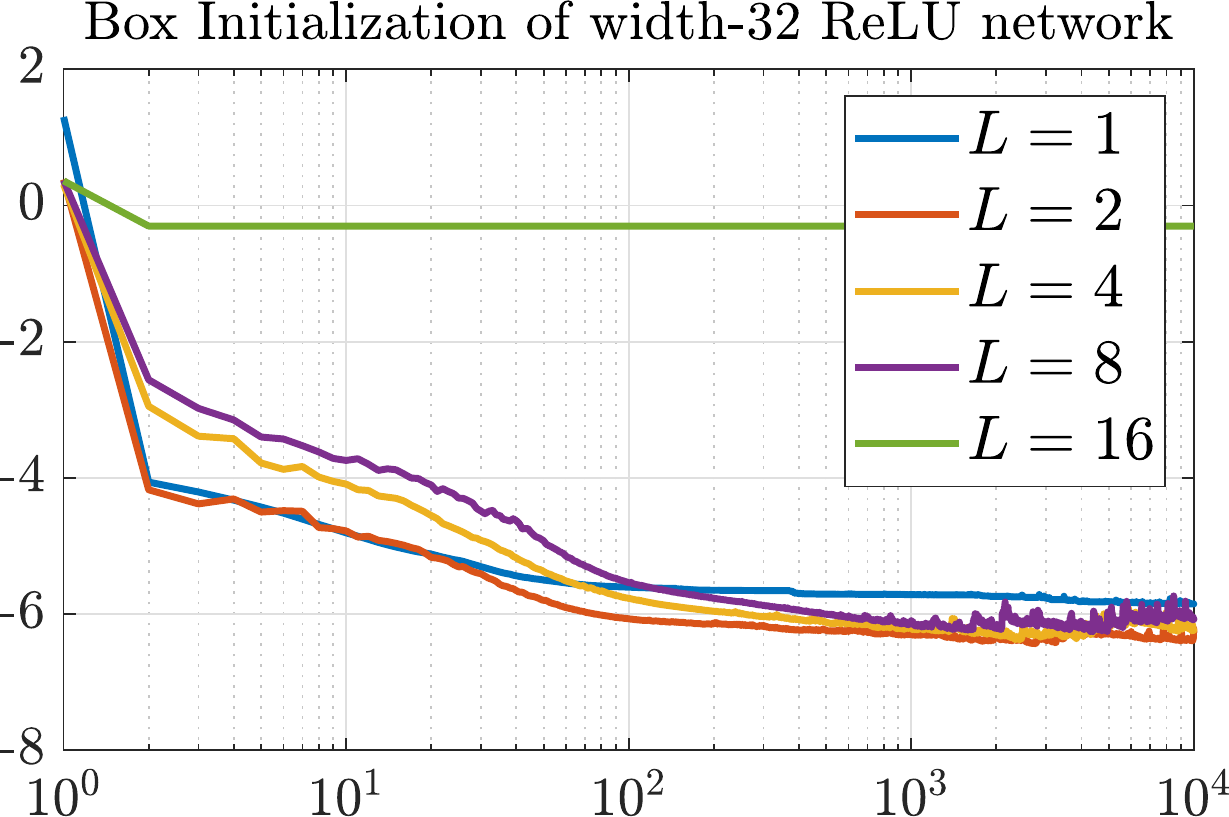}
    \vspace{-2ex}
    \caption{Mean of $\log_{10}(\text{Loss})$ over 16 training runs of plain width-32 ReLU network with $L = 1, 2, 4, 8$ and $16$ hidden layers for the He \emph{(left)} and Box \emph{(right)} initializations. The learning rate is 0.005 throughout.}
    \vspace{-2ex}
    \label{fig:relu_shallow_initialization}
\end{figure}

To understand why this occurs, consider
the image $P_L$ of the unit box $[0,1]^w$ under $L$ hidden layers of the network, excluding the $d$-dimensional input layer for now. 
Fig. \ref{fig:plain_image} shows the evolution of $P_L$ through each layer for different initialization approaches.
Because each hidden layer does not map $[0,1]^w$ \emph{onto} $[0,1]^w$, as the number of layers increases, we expect $P_L$ to shrink, lose dimension, and eventually collapse to a point. 
In turn, for input dimension $d \le w$, the image of the input box $[0,1]^d$ is a submanifold of $P_L$, given by the parametrization $(\Phi_1(\bm{x}), \Phi_2(\bm{x}), ... \Phi_w(\bm{x}))$ for $\bm{x} \in [0,1]^d$. 
For example, for the DNNs shown in Fig. \ref{fig:plain_image}, with a one-dimensional input this submanifold would be a curve within $P_L$.
The basis function $\Phi_i$ is the projection of this submanifold onto the $i$th coordinate axis\edit{; this is illustrated for a width $w=2$ network with input dimension $d = 1$ in Fig. \ref{fig:box_mapping}.} Once the image $P_L$ is a point, this submanifold within $P_L$ is also a point, so all initial $\Phi_i$ will be constant. Fig. \ref{fig:plain_image} demonstrates this for a width-two ReLU network; the He, Glorot, and box initialization suffer from this flaw in the plain network case. While the growth in the magnitude of the basis is controlled (as expected) by box initialization, and the support of the basis does not collapse as quickly in this instance, a statistical study of this approach will indicate that the collapse to a point for all three initializations is inevitable. One possible treatment of this collapse has been proposed in~\cite{lu2019dying}. \edit{Issues of training DNNs have also been discussed by \citet{hanin2018start}, who proposed a scaling of depth to width as a possible solution.}   
Next, we illustrate how ResNets avoid this issue at higher depth, and propose an analogous box initialization for ResNets.

\begin{figure}[t!]
    \centering
    \includegraphics[width=0.66\textwidth]{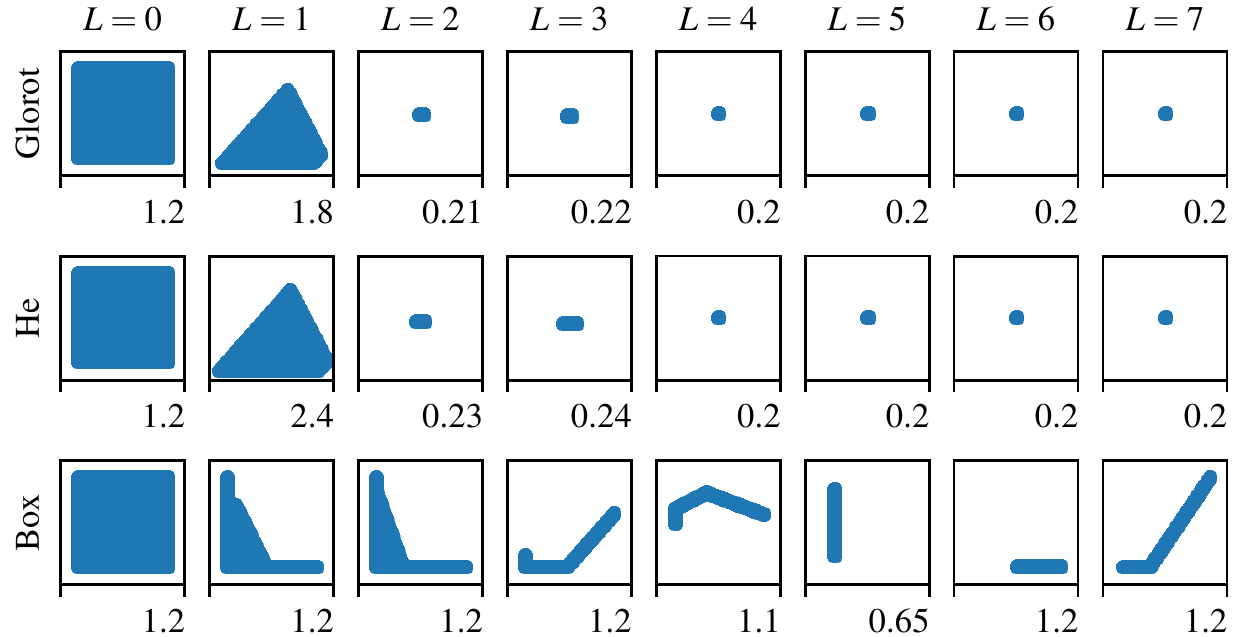}
        \vspace{-3ex}
    \caption{
Images $P_L$ of the unit square $[0,1]^2$ under $L$ initialized hidden layers of plain networks for He (\textit{top}) and Box (\textit {bottom}) initializations. 
    Values are presented on the square $\left[-0.2,H\right]^2$, where $H$ is denoted to the bottom-right of each image. Collapse to a point corresponds to constant basis functions.
} 
    \label{fig:plain_image}
\end{figure}

\begin{figure}
    \centering
    \includegraphics[width=0.75\textwidth]{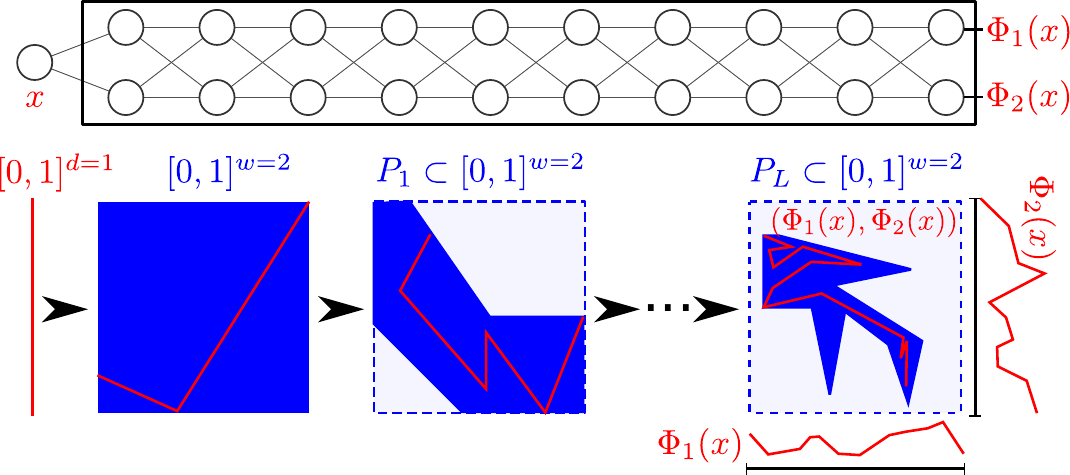}
        \vspace{-3ex}
    \caption{
    \edit{Artist's depiction of the $d$-dimensional manifold (\emph{red}) parametrized by $(\Phi_1(\bm{x}), ..., \Phi_w(\bm{x}))$, which is the image of the input domain $[0,1]^d$ under the input and hidden layers, as a submanifold of the image $P_L$ (\emph{blue}) of the unit box $[0,1]^w$ under the hidden layers. Here $d=1$ and $w=2$ to make visualization possible.}
    } 
    \label{fig:box_mapping}
\end{figure}


\subsection{Residual Neural Networks (ResNets)}
\label{resnet_init_subsec}
Consider a residual neural network with input dimension $d$ and hidden layer width $w$. \edit{As usual for a ResNet, unless $d=w$, the first hidden layer is initialized as plain layer as described in Section \ref{plain_init_subsec} above. Then, for the remaining hidden layers, } to initialize the neuron $i$, $1 \le i \le w$,
\begin{enumerate}[noitemsep,nolistsep]
\item
For $m$ specified later, select $\bm{p} \in [0,m]^w$ at random. 
\item
Select a unit normal $\bm{n}$ at $\bm{p}$ with random direction.
\item
For $\delta$ specified later, choose a scaling $k$ such that 
\begin{equation}
\underset{[0,m]^w}{\text{max }}\sigma(k(\bm{x}-\bm{p})\cdot \bm{n}) = \delta m.
\end{equation}
We again apply Lemma \ref{maurosTrick} to find the maximal corner.
\item
Row $\bm{w}_i$ of $\bm{W}^{\bm{\xi}}$ and $\bm{b}^{\bm{\xi}}$ is selected
as $b_i = k\bm{p}\cdot\bm{n}$ and $\bm{w}_i = k\bm{n}^T$.
\end{enumerate}
As for the plain DNN initialization, a more detailed description of the weight and bias initialization procedure can be found in Algorithm~\ref{alg:dense-resnet} in Appendix \ref{algorithm_appendix}.
With the layer initialized as above, consider feeding a box $[0,m]^w$ as input into a given layer.
For a residual neural network, the output $\bm{x}^{l+1}$ of layer $l>1$ is given by 
$\bm{x}_{l+1} = \bm{x}_{l} + \sigma({\bm{W}_l} \bm{x}_{l} + \bm{b}_{l})$ while for the first layer we have $\bm{x}_{2} = \sigma({\bm{W}_1} \bm{x}_{1} + \bm{b}_{1})$.
Then, we have for every component $i \in \{ 1, 2, \hdots, w\}$, $l>1$
\begin{align}
\underset{\bm{x}_l \in [0,m]^w}{\text{min }} (\bm{x}_{l+1})_i &\ge
\underset{\bm{x}_l \in [0,m]^w}{\text{min }} (\bm{x}_l)_i + \underset{\bm{x}_l \in [0,m]^w}{\text{min }} 
\sigma(k(\bm{x}_l-\bm{p})\cdot \bm{n})
\ge \underset{\bm{x}_l \in [0,m]^w}{\text{min }} (\bm{x}_l)_i \ge 0 \\
\underset{\bm{x}_l \in [0,m]^w}{\text{max }} (\bm{x}_{l+1})_i &\le 
\underset{\bm{x}_l \in [0,m]^w}{\text{max}} (\bm{x}_l)_i + \underset{\bm{x}_l \in [0,m]^w}{\text{max }}
\sigma(k(\bm{x}_l-\bm{p})\cdot \bm{n}) 
\le m + m\delta .  
\end{align}
Thus, layer $l$ maps $[0,m]^w$ into $[0,m(1+\delta)]^w$ permitting some growth specified by $\delta$. Assuming the input into the first hidden layer is contained in $[0,1]^w$, initializing the hidden layers with $\delta = \frac{1}{L}$ leads to a network that maps
\begin{equation}
[0,1]^d 
\xrightarrow[]{\text{into}}
[0,1]^w 
\xrightarrow[]{\text{into}}
\left[0,1+\frac{1}{L}\right]^w 
\xrightarrow[]{\text{into}}
\left[0,\left(1+\frac{1}{L}\right)^2\right]^w 
\xrightarrow[]{\text{into}}
\hdots
\xrightarrow[]{\text{into}}
\left[0,\left(1+\frac{1}{L}\right)^{L-1}\right]^w.
\end{equation}
This implies the final output of the hidden layer is contained in the box $[0, e]^w$; in other words, the values of each basis function are contained in $[0,e]$. 
Thus, \edit{we apply steps (1) -- (4)} with parameters 
\begin{equation}
\delta=1 \text{ and } m=1 \text{ for } l=1; \quad \delta = \frac{1}{L} \text{ and }
m = \left( 1 + \frac{1}{L} \right)^{l-1} \text{for } l > 1,
\end{equation}
\edit{and refer to this as the box initialization for ResNets.}

An interesting observation regarding the \edit{box initialization for ResNets} is its connection to the recently
developed ODE\edit{-}based \edit{DNN} architectures of ~\cite{haber2017stable} and~\cite{chen2018neural}. In those \edit{works}, an $O\left(\frac{1}{L}\right)$ \edit{temporal} step size scales the activation function, where $L$ is the number of \edit{time steps}. This ensures that the growth of the network features is a function of the length of the time interval (assuming bounded weights and biases). This is \edit{analogous} to \edit{the properties of} the \edit{ResNet box} initialization \edit{shown above}. An important difference\edit{,} however, is that the \edit{ODE-based} architectures retain \edit{this} scaling \edit{throughout} the training process. 

We compare the use of the box initialization for a residual neural network with hidden layer width 32 against the He initialization in Fig. \ref{fig:relu_deep_initialization} for approximating $\sin(2\pi x)$ on $[0,1]$. We average over 16 independent runs. The box initialized basis is again richer than the He basis and yields an initial LS loss consistently 4 orders of magnitude lower. The loss during training exhibits similar improvements over the He basis. At 128 layers, it is now the He basis which fails to train. 
\begin{figure}
    \centering
    \includegraphics[height=1.8in]{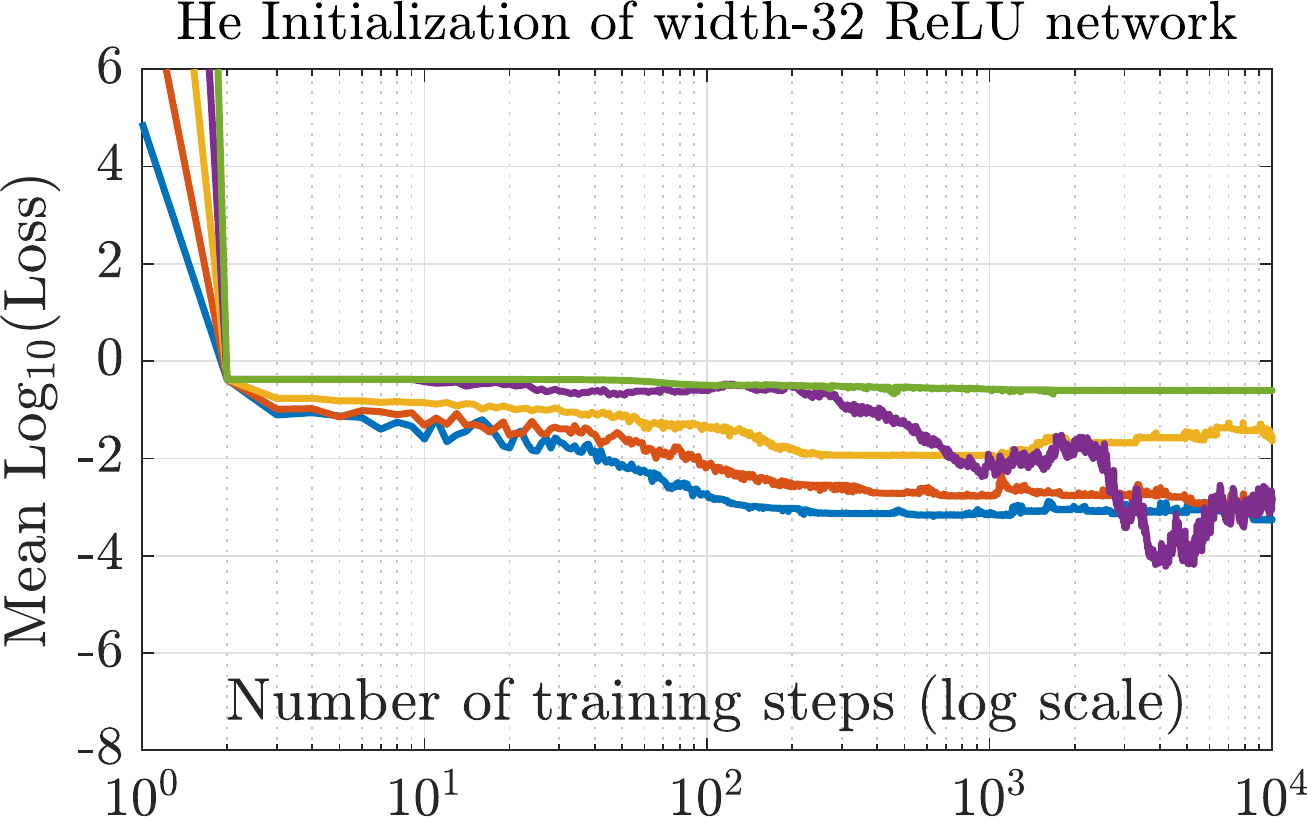}
    \includegraphics[height=1.8in]{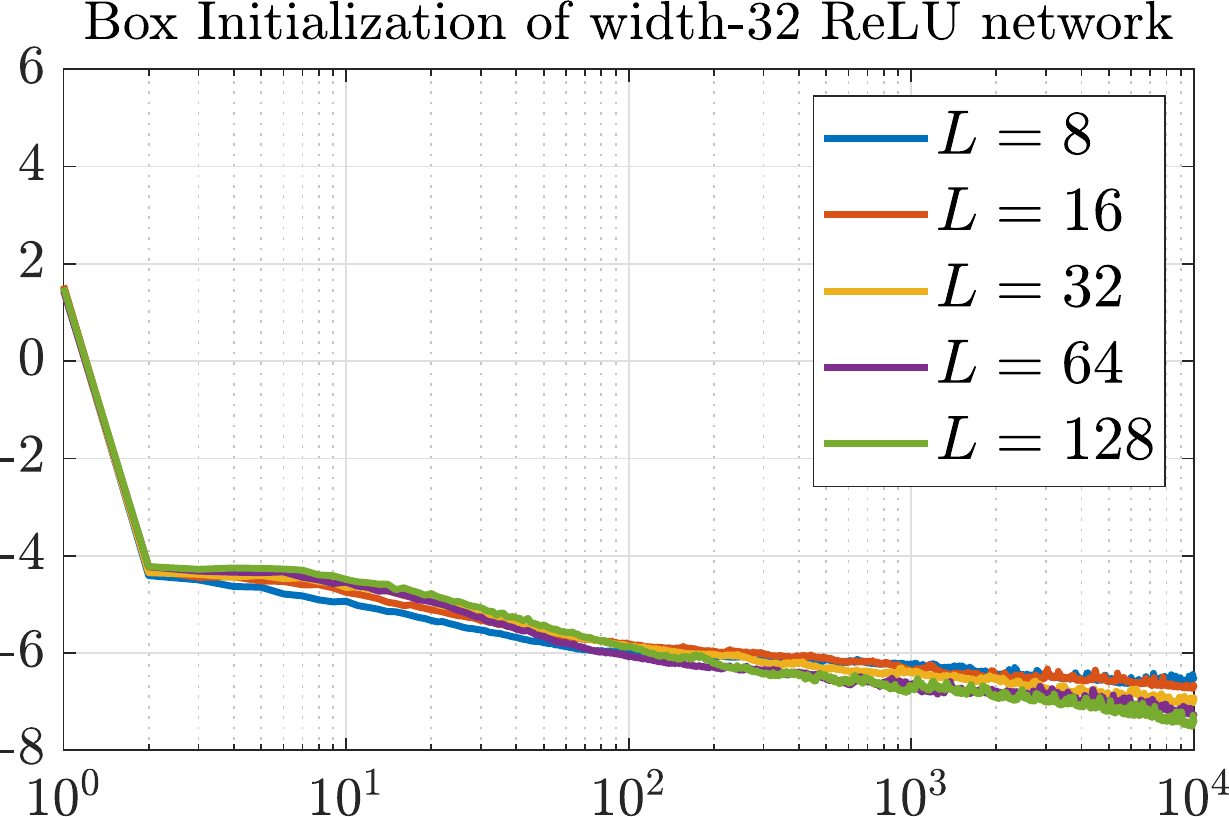}
        \vspace{-2ex}
    \caption{Mean of $\log_{10}(\text{Loss})$ over 16 training runs of residual width-32 ReLU network with $L = 8, 16, 32, 64$ and $128$ hidden layers and training rate $2^{-(k+3)}$ for the He \emph{(left)} and Box \emph{(right)} initializations.}
    \label{fig:relu_deep_initialization}
\end{figure}

The advantages of the box initialization over the He initialization can be illustrated for a width-2 network by again studying the image \edit{$P_L$} of the unit square $[0,1]^2$ under both initialization in Fig. \ref{fig:resnet_image}. Note that the image of the square never collapses to a point due to the ResNet architecture, regardless of initialization. Hence, the initialized basis will not consist of constant functions. This is a new interpretation of the stability provided by residual neural networks\edit{; for other perspectives, see \citet{hanin2018start}, \citet{haber2017stable}, and \citet{he2016deep}.} Nevertheless, both the Glorot and the He initialization exhibit different pathologies in the ResNet case as depth increases: blow-up of the basis function magnitudes and convergence of the image \edit{$P_L$} to lines through the origin. The latter property implies linearly dependent basis functions $\phi_1 = C\phi_2$, again resulting in a decreased expressive power of the initialized basis. All of these properties are illustrated in the basis function plots in Fig. \ref{fig:basis_plain} in Appendix \ref{basis_appendix}. The ResNet box initialization, however, exhibits both the boundedness of \edit{$P_L$} proven above and a remarkable preservation of the area of \edit{$P_L$} as depth increases. We have not yet found an explanation for the latter property, but these results explain 
the benefits of the box-initializion for deep networks observed in Fig.~\ref{fig:relu_deep_initialization}. 

We observe similar properties of the ResNet Box initialization in higher dimensions as well. In Fig.~\ref{fig:resnet_image} we examine the eigenvalues of the covariance of the image of a set of input points sampled from $\mathcal{U}[0,1]^w$ through networks of increasing depth. We find that for the Glorot and He initializations, the ratio between the smallest and largest eigenvalues quickly become zero with increasing depth. This suggests that one basis function becomes linearly dependent upon the others with only a few layers. Worse, the ratio between the second largest and the largest eigenvalues eventually becomes zero, suggesting that the basis functions all become linearly dependent. In contrast, neither ratio tends toward zero for the Box initialization, indicating that the basis functions remain independent, even for very deep networks.

%% file: lsgd_appendix.tex
\label{lsgd_appendix}
We provide here some supplemental results providing additional insight into the properties and advantages of LSGD and the computational cost relative to GD. We consider first a toy 2D problem in Figure \ref{fig:LSQRcartoon}, \edit{where we compare GD to LSGD for minimizing the loss $5x^2 - 6xy + 5y^2$. This function is quadratic in both $x$ and $y$, but to make an analogy to \eqref{eqn:separableNonLinearLeastSquare} we take the $x$-direction to correspond to the linear activation variable $\bm{\xi}^L$ and the $y$-direction to the hidden variable $\bm{\xi}^H$.} We can visualize explicitly that LSGD realizes the global \edit{minimum} \edit{in $x$} at each step, and thus approaches the \edit{global} minimum \edit{in $(x,y)$} along a trajectory $(x_k, y_k)$ where the \edit{the coordinate $x_k$ always satisfies the least squares problem $x_k = LS(y_k)$ and is ``optimal'' for the coordinate $y_k$.}
\begin{figure}[htpb!]
    \centering
    \includegraphics[width=0.45\textwidth]{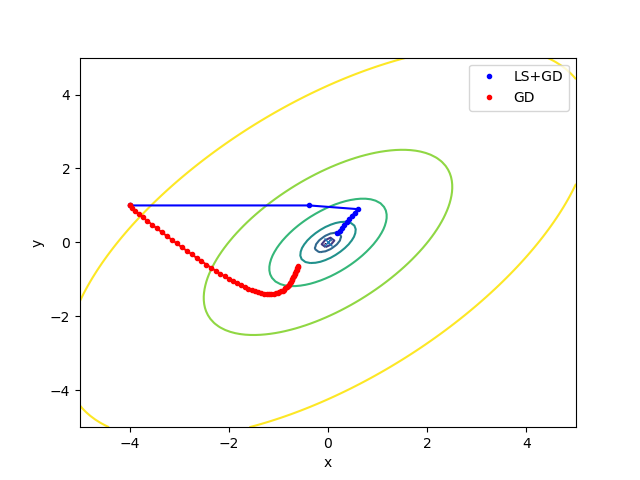}        \includegraphics[width=0.45\textwidth]{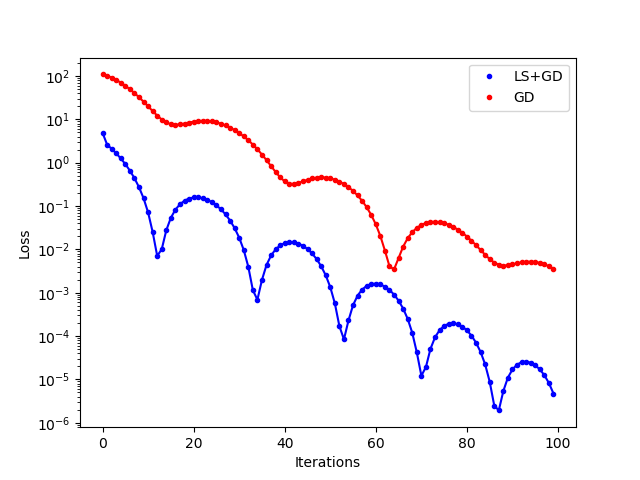}    
    \caption{\edit{\emph{Left:} Paths $(x_k,y_k)$ taken by LSGD and GD} to minimize the function $5 x^2 - 6 x y + 5 y^2$, for learning rate of $0.1$ and initial guess of $(x,y) = (-4,1)$. \edit{Least squares optimiz{\edit{ation}} corresponds to finding  \edit{the} global \edit{minimum} in coordinate \edit{$x$ for} fixed $y$ at each step.
    \edit{Note that for LSGD, after the initial least squares solve, each plotted $(x_k,y_k)$ is the result of gradient descent followed by least squares, rather than either of these steps individually. \emph{Right:} LSGD achieves lower loss for the same number of iterations as GD.}}
    }
    \label{fig:LSQRcartoon}
\end{figure}

We next provide in Figure \ref{fig:lsgd_path} a sketch explaining how the LSGD approach may offer gains due to the fact that the dynamics of training are constrained to follow a manifold $\xi^L = LS(\xi^H)$ which necessarily contains all local minima. This figure also makes clear that the paths of GD training and LSGD training are not comparable globally. While training on this manifold may be more stable and lead to faster training, nothing precludes the existence of barriers along this manifold between an initial condition and a ``good'' local minima, which may be bypassed by GD training, as alluded to in Section \ref{LS_section} during the discussion Figure \ref{fig:lsgd_training}.
\edit{Figure \ref{fig:lsgd_path} also illustrates that LSGD can be viewed as type of coordinate descent method \citep{nocedal2006numerical} in which steps in $\bm{\xi}^L$ are taken until a global minimum is reached before the variables are alternated, although we find a global minimum in one shot with a least squares solver.} 
We \edit{also} conjecture that there may be interesting connections \edit{with} the dynamical system interpretation of training ResNets \citep{haber2017stable, chen2018neural} and work on fast/slow manifold dynamics \citep{gear2005projecting}.
\begin{figure}[t!]
    \centering
    \includegraphics[width=0.5\textwidth]{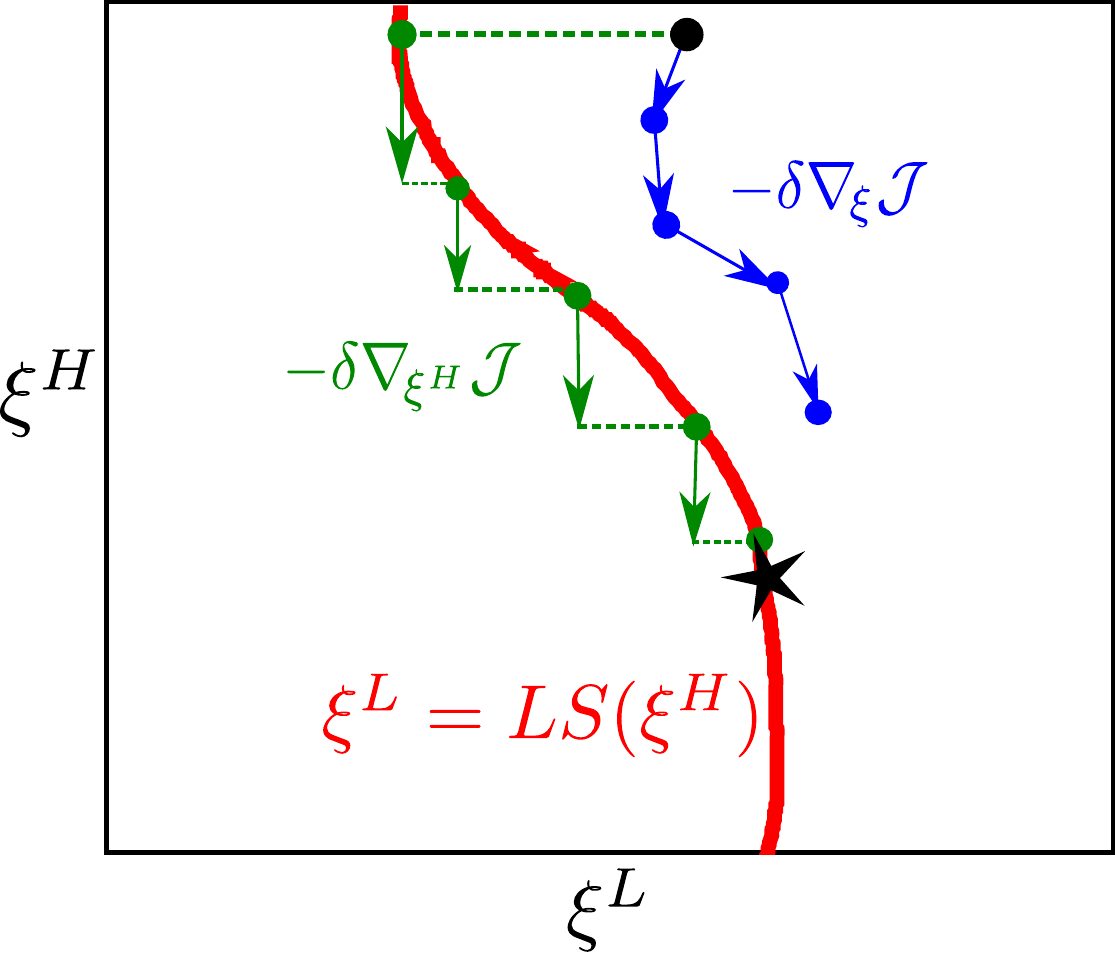}
    \vspace{-2ex}
    \caption{Artist's depiction of the LSGD algorithm. The black dot denotes the initial guess and the black star a minimum that the user wants train the neural network to. The red line represents the submanifold in parameter space $(\bm{\xi}^H, \bm{\xi}^L)$ for which $\bm{\xi}^L$ is a solution to the least squares problem for fixed $\bm{\xi}^H$. Note that because the local minimum illustrated by the black star must also be a global minimum in $\bm{\xi}^L$, it must lie on the manifold $\bm{\xi}^L = LS(\bm{\xi}^H)$ illustrated by the red line. 
    Note also that 
    $\nabla_{\bm{\xi}} \mathcal{J} = 
    (\nabla_{\bm{\xi}^H} \mathcal{J}, \bm{0})$ on the manifold $\bm{\xi}^L = LS(\bm{\xi}^H)$.
    The \edit{blue} curve represent\edit{s} a path of the GD method, while the \edit{rectilinear} green \edit{curve} a path of LSGD. An initial least squares solve (dash\edit{ed} green line) \edit{moves} the neural network parameters to the submanifold $\bm{\xi}^L = LS(\bm{\xi}^H)$. In the LSGD algorithm, all gradients are computed from this manifold. \edit{Each step of} gradient descent \edit{can move the parameters} off this manifold, \edit{but} the least squares solve \edit{that follows} will project back onto the \edit{manifold}. 
    }
    \label{fig:lsgd_path}
\end{figure}

The computational cost of including the least squares step for \edit{$\bm{\xi}^L$} after each gradient descent step in \edit{$\bm{\xi}^H$} depends heavily on the implementation details of both steps -- for example, the specific least squares solver, whether GPU acceleration is used for gradient descent, memory access pattern used to overwrite linear layer variables, etc. Generally, the least squares solve only increases with the width of the network ($O(W^3)$ for dense solvers), whereas the gradient descent step increases with both the width and the depth, as indicated by Table \ref{complexity_table}. 
\begin{table}[htpb!]
\centering
\begin{tabular}{|c|l|l|l|l|l|}
\hline
                                & \multicolumn{5}{c|}{\textbf{Depth (hidden layers)}}                   \\ \hline
\multirow{5}{*}{\textbf{Width}} & \textbf{}    & \textbf{4} & \textbf{16} & \textbf{64} & \textbf{256} \\ \cline{2-6} 
                                & \textbf{4}   & 1.67       & 1.60        & 1.43        & 1.36         \\ \cline{2-6} 
                                & \textbf{16}  & 2.08       & 1.50        & 1.44        & 1.41         \\ \cline{2-6} 
                                & \textbf{64}  & 1.68       & 1.37        & 1.37        & 1.33         \\ \cline{2-6} 
                                & \textbf{256} & 2.00       & 1.65        & 1.52        & 1.49         \\ \hline
\end{tabular}
\caption{Relative increase in wall time for 1000 iterations of LSGD vs pure GD (using the Adam optimizer) for a plain ReLU network, obtained using CPU implementation of Tensorflow on an Intel i7-8700K processor. For deeper networks, the increase is smaller since the computational cost of gradient descent, unlike that of least squares, grows with the depth and dominates the wall time.}
\label{complexity_table}
\end{table}

%% file: algorithms_appendix.tex
\label{algorithm_appendix}
We provide in Algorithms \ref{alg:dense-init} and \ref{alg:dense-resnet} concise definitions of the box initialization algorithm for both plain and ResNet DNNs, respectively. 

\edit{Algorithm \ref{alg:dense-init} initializes layer $l$, and takes as input the dimension of the input to this layer, i.e., the width $w_{l-1}$ of the previous layer, and the output dimension $w_l$. The main points of this algorithm are outlined in Section \ref{plain_init_subsec}.
Note that the random normal vector with uniform random direction $\hat{\bm{n}}$ is conveniently sampled (lines 3 -- 4) by sampling from a isotropic multivariate normal of mean $\bm{0}$ and then normalizing. 
Once $\bm{p}$ and $\hat{\bm{n}}$ have sampled, a cut plane for a ReLU function is defined. To compute the scaling constant $k$ in Section \ref{plain_init_subsec} such that the maximum of the ReLU function on $[0,1]^{w_{l-1}}$ is $1$, it is necessary to locate the furthest corner $p_{\text{max}} \in [0,1]^{w_{l-1}}$ in the direction of $n$ from the cut plane of the ReLU function. To do this efficiently, we provide a closed form expression in line 5 for the corner of the box where the maximum occurs; this formula is proven in Lemma \ref{maurosTrick}. The scaling factor $k$ is then the inverse of the distance of the cut plane to this corner. 
}

\begin{algorithm}
\caption{Plain Network Box Initialization}
\begin{algorithmic}[1]
    \Function{PlainInit}{$w_{l-1},w_{l}$}
        \State $\bm{p} \sim \mathcal{U}[0,1]^{w_{l-1} \times w_l}$\Comment{Sample $w_l$ points in $[0,1]^{l-1}$}
        \State $\hat{\bm{n}} \sim \mathcal{N}[0,1]^{w_{l-1} \times w_l}$\Comment{Sample from a normal distribution}
        \State $n_{ij} = \hat{n}_{ij}/||\hat{\bm{n}}_j||_2^2$\Comment{$w_{l}$ random unit vectors of dimension $w_{l-1}$}
        \State $p_{\max} = \max(0,\mathrm{sign}(n_{ij}))$
        \State $k_j=1/\sum_i \left((p_{\max} - p_{ij})n_{ij}\right)$
        \State $A_{ij}^l=k_j n_{ij}$
        \State $b_i^l = \sum_j{k_j n_{ij} p_{ij}}$
        \State \Return $\bm{A}^l,\bm{b}^l$
    \EndFunction
\end{algorithmic}
\label{alg:dense-init}
\end{algorithm}

\begin{lemma}\label{maurosTrick}
Let $\mathbb{H}$ be a $(d-1)$ dimensional hyperplane in $\mathbb{R}^d$ and let $\bm{n}$ be a normal to $\mathbb{H}$. Then, the maximum distance along direction $\bm{n}$ from $\mathbb{H}$ and any point in the unit hypercube $[0,1]^d$ is achieved on 
\begin{equation}
( \textup{max}(\textup{sgn}(n_1), 0),  \textup{max}(\textup{sgn}(n_2), 0), \hdots, \textup{max}(\textup{sgn}(n_d), 0) ) = \textup{max} (\textup{sgn}(\bm{n}), \bm{0}). 
\end{equation}
\end{lemma}
\begin{proof}
Let us refer to the distance in question as the \emph{directed distance}. The maximum directed distance is achieved on a corner of  $[0,1]^d$, not necessarily unique. Let $C^*$ be such a corner; the fact that $C^*$ maximizes the directed distance from $\mathbb{H}$ is invariant under parallel transport of the hyperplane $\mathbb{H}$ in direction $\bm{n}$. Parallel transport $\mathbb{H}$ in direction $\bm{n}$ until the plane lands on $C^*$; then every point in $[0,1]^d$ is either on $\mathbb{H}$ or on the opposite of $\mathbb{H}$ from $\bm{n}$. Let $C^*_i \in \{ 0, 1 \}$ denote the coordinates of $C^*$. Make $C^*$ the origin. Consider the $d$ unit vectors $\bm{v}_i$ from $C^*$ to the other $d$ corners of $[0,1]^d$ along one of the axes. If the coordinate $C^*_i$ had been $1$, then $\bm{v}_i = -\bm{e}_i$; else if $C^*_i$ had been $0$, then $\bm{v}_i = \bm{e}_i$. Since the other corners are separated from $\bm{n}$ by $\mathbb{H}$, we have that
\begin{equation}
0 \ge \langle \bm{n}, \bm{v}_i \rangle = 
\begin{cases}
n_i \text{ if } C^*_i = 0 \\
-n_i \text{ if } C^*_i = 1. 
\end{cases}
\end{equation}
Hence if $0 > n_i$, then $C_i^* = 0$, while if 
$0 < n_i$, then $C_i^* = 1$. If $n_i = 0$, then both the corner
$C^*$ and that corner with the bit $C^*_i$ flipped achieve the same 
directed distance from $\mathbb{H}$, so we may
take $C^*_i = 0$.
\end{proof}

\edit{Algorithm \ref{alg:dense-resnet} follows from the outline in Section \ref{resnet_init_subsec} in a similar way.
We initialize the first hidden layer as a plain hidden layer using Algorithm \ref{alg:dense-init} above (this is necessary for $d \neq w$). Then $p_{\text{max}}$ is found by applying the same algorithm in Lemma \ref{maurosTrick} and scaling by the constant $m$ to yield a corner in the box $[0,m]^{w_{l-1}}$. The scaling constant $k$ now includes the factor $\frac{1}{L-1}$.}

\begin{algorithm}
\caption{ResNet Box Initialization}
\begin{algorithmic}[1]
    \Function{ResNetInit}{$w_{l-1},w_{l},L$}
        \If {$l==1$}
            \State \Return PlainInit$(w_{l-1},w_{l})$
        \Else
            \State $m = \left(1+1/(L-1)\right)^l$
            \State $\bm{p} \sim \mathcal{U}[0,m]^{w_{l-1} \times w_l}$
            \State $\hat{\bm{n}} \sim \mathcal{N}[0,1]^{w_{l-1} \times w_l}$
            \State $n_{ij} = \hat{n}_{ij}/||\hat{\bm{n}}_j||_2^2$
            \State $p_{\max,ij} = m \max(0,\mathrm{sign}(n_{ij}))$
            \State $k_j=1/\sum_i \left( (p_{\max,ij} - p_{ij})n_{ij}(L-1)\right)$
            \State $A_{ij}^l=k_j n_{ij}$
            \State $b_i^l = \sum_j{k_j n_{ij} p_{ij}}$
            \State \Return $\bm{A}^l,\bm{b}^l$
        \EndIf
    \EndFunction
\end{algorithmic}
\label{alg:dense-resnet}
\end{algorithm}

%% file: basis_function_appendix.tex
\label{basis_appendix}
Figures~\ref{fig:basis_plain} and \ref{fig:basis_resnet} show the basis functions at initialization for Plain and ResNet architectures, respectively, in one-dimension.

\begin{figure}[htpb!]
    \centering
    \includegraphics[width=\textwidth]{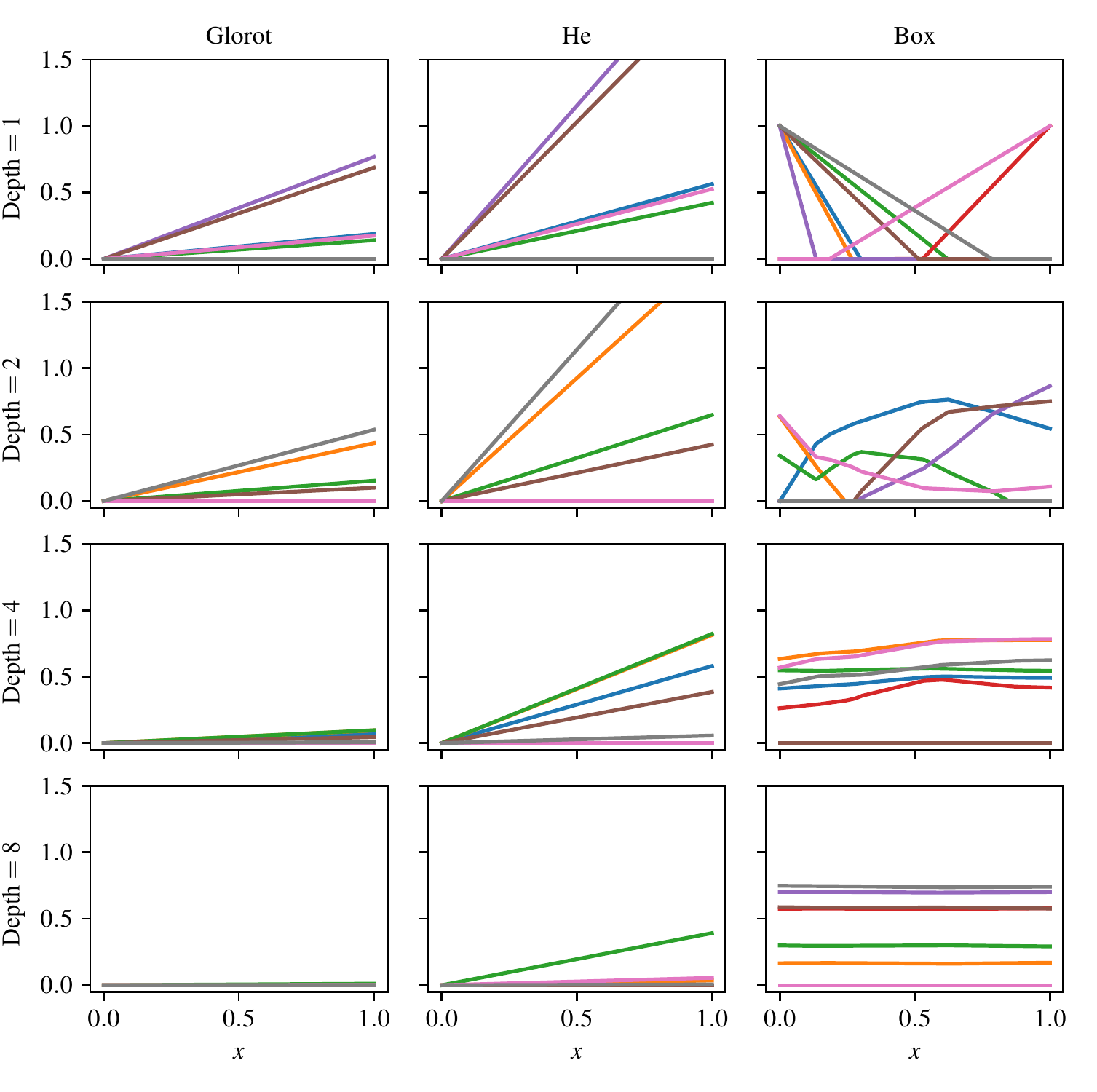}
    \vspace{-2ex}
    \caption{\textbf{Plain network basis} functions, $\Phi_i$, after initialization for Glorot (\emph{left column}), He (\emph{middle column}), and Box (\emph{right)} initializations with increasing depth and width 8. The input is one-dimensional. As discussed in Section \ref{plain_init_subsec}, these figures illustrate that the Box initialized basis is richer in features than the He and Glorot initialized bases, but suffers from ``collapse'' to constant functions as depth increases (notice that this tendency is also visible for the He and Glorot basis functions).}
    \label{fig:basis_plain}
\end{figure}

\begin{figure}[htpb!]
    \centering
    \includegraphics[width=\textwidth]{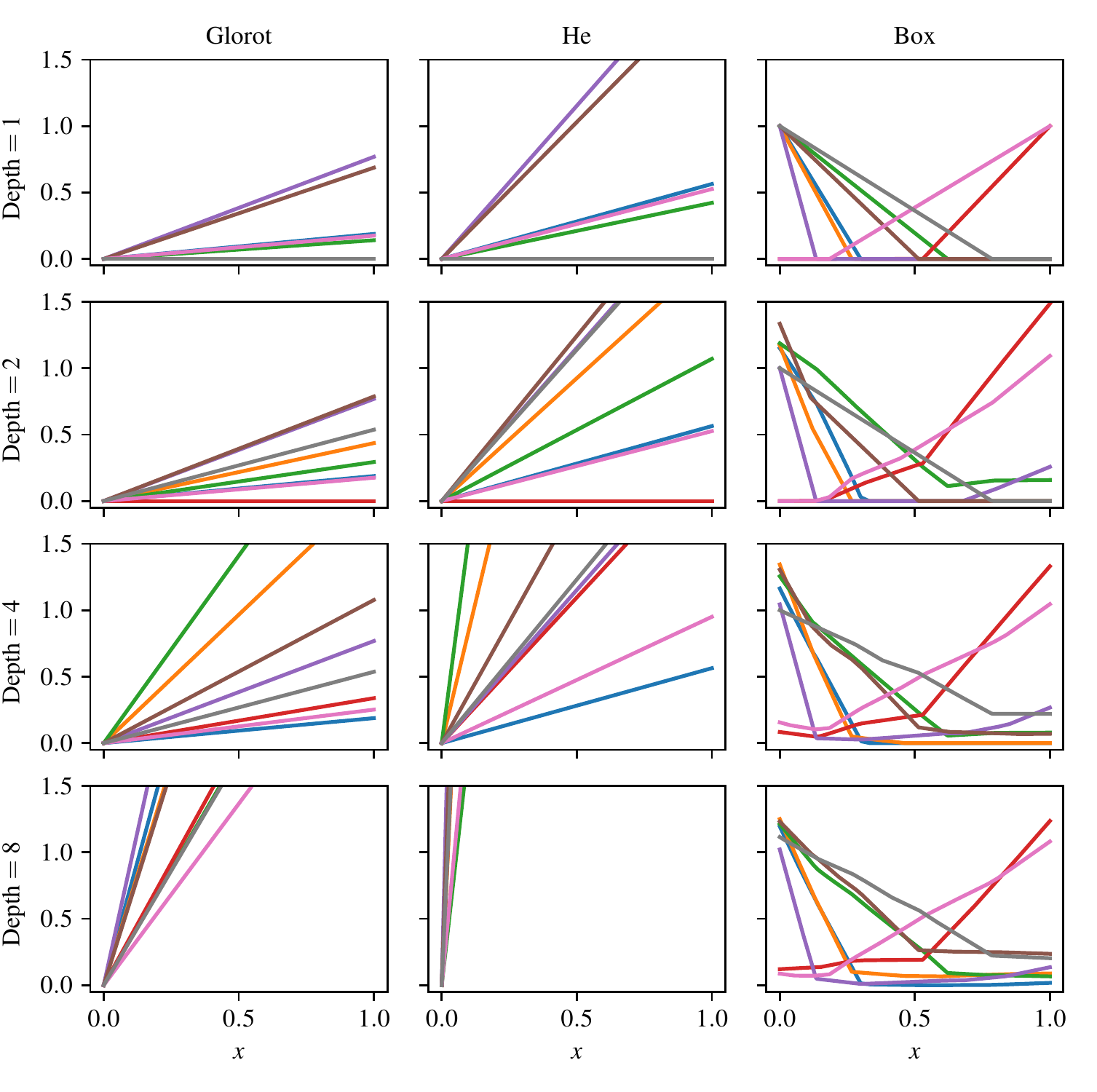}
    \vspace{-2ex}
    \caption{\textbf{ResNet basis} functions, $\Phi_i$, after initialization for Glorot (\emph{left column}), He (\emph{middle column}), and Box (\emph{right)} initializations with increasing depth and width 8. The input is one-dimensional. As discussed in Section \ref{resnet_init_subsec}, these figures illustrate that for ResNets, the box initialization consistently (with depth) produces basis functions with more features in the input domain than the He and Glorot initializations. The box initialized basis no longer suffers from the collapse to a constant basis as for plain architectures, nor does it exhibit the blow-up evident in the He basis, as depth increases.}
    \label{fig:basis_resnet}
\end{figure}

%% file: pinns_snapshots.tex
\section{PINNs snapshots}\label{pinnsSnapshots}
The images below depict the PINN solution to the constant coefficient transport equation at training step $i$ with the cut planes of the ReLU basis superimposed as dashed red lines. These training snapshots demonstrate that the LSGD trained PINN (right column) finds the correct characteristics of the PDE with ReLU cutplanes far faster than the GD trained PINN (left column). Note that the $i$'s are different in the two columns, and both networks have identical initializations. 
\begin{figure*}[htbp!]
   \centering
\begin{tabular}{c|c}
{\Large Gradient Descent}
&
{\Large LSGD} \\
\hline
\\
\includegraphics[width=1.5in]{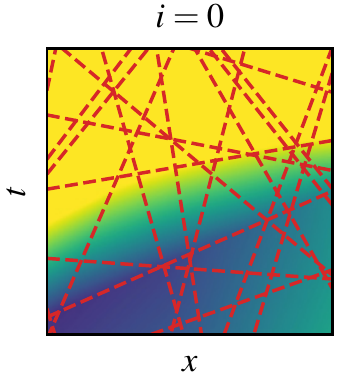} 
\includegraphics[width=1.5in]{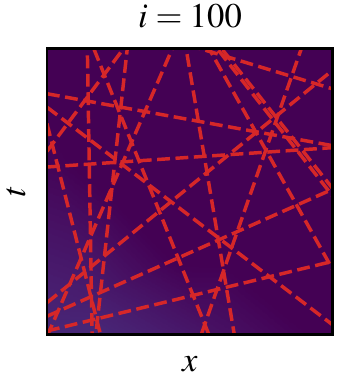}
&
\includegraphics[width=1.5in]{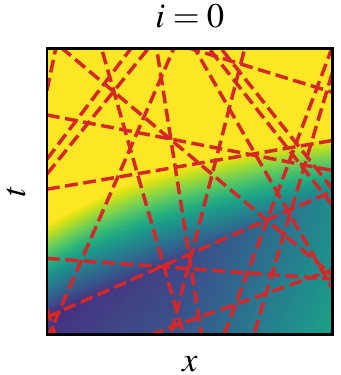}
\includegraphics[width=1.5in]{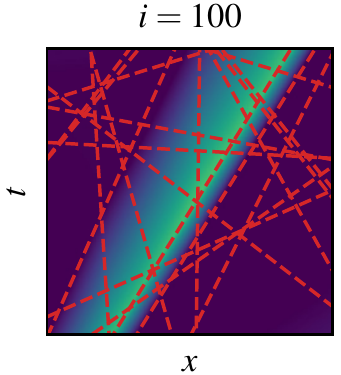}
\\
\includegraphics[width=1.5in]{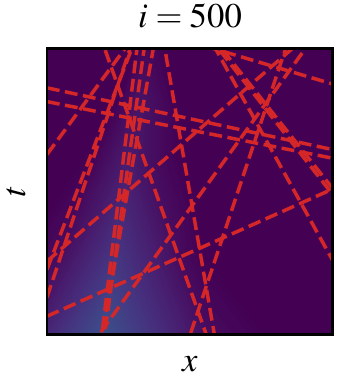} 
\includegraphics[width=1.5in]{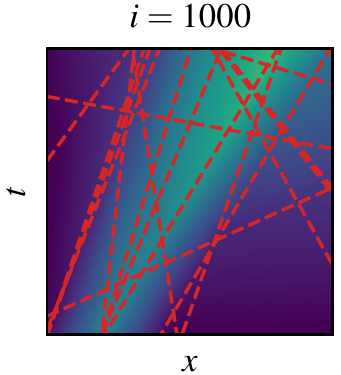} 
&
\includegraphics[width=1.5in]{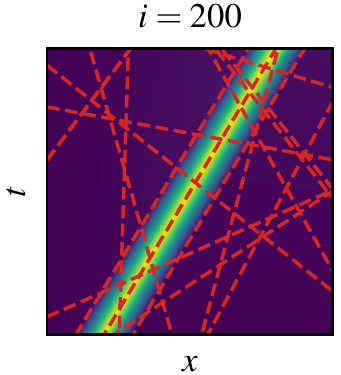}
\includegraphics[width=1.5in]{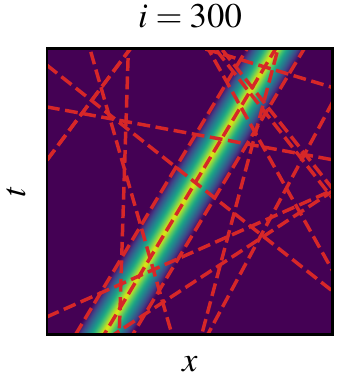}
\\
\includegraphics[width=1.5in]{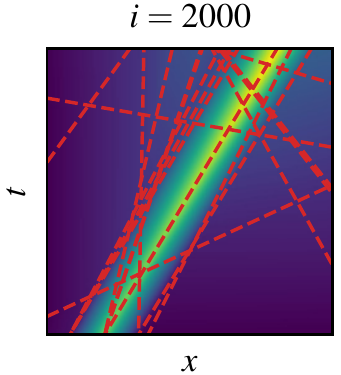} 
\includegraphics[width=1.5in]{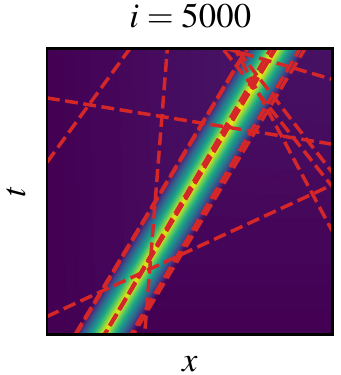} 
&
\includegraphics[width=1.5in]{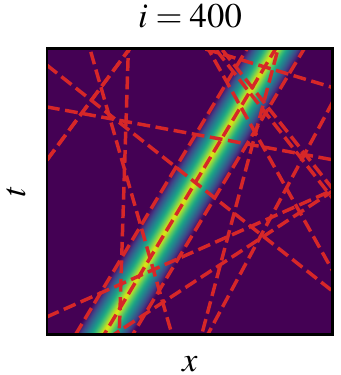}
\includegraphics[width=1.5in]{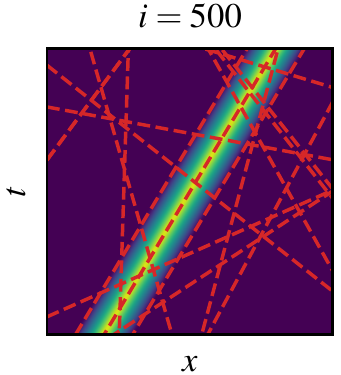}
\end{tabular}
\end{figure*}

